\documentclass[letterpaper,journal]{IEEEtran}
\usepackage[utf8]{inputenc}
\usepackage[T1]{fontenc}
\usepackage{amsmath,amssymb,amsfonts,amsthm}
\usepackage{graphicx}
\usepackage{subcaption}
\usepackage{booktabs,array,multirow,tabularx,threeparttable}
\usepackage{microtype}
\usepackage{xcolor}
\usepackage{bm}
\usepackage{textcomp}
\usepackage{url}

\usepackage[numbers,sort&compress]{natbib}
\usepackage[hidelinks]{hyperref}
\newcommand{\abs}[1]{\lvert#1\rvert}

\newtheorem*{theorem*}{Theorem}
\newtheorem{theorem}{Theorem}
\newtheorem{definition}{Definition}
\newtheorem*{lemma*}{Lemma}
\newtheorem{lemma}{Lemma}
\newtheorem{remark}{Remark}

\newtheorem{assumption}{Assumption}

\begin{document}

\title{Fine-Grained Analysis of Catastrophic Forgetting in Continual Learning via Algorithmic Stability}

\author{Guodong Zheng\textsuperscript{\rm 1},
        Yan Sun\textsuperscript{\rm 2},
        Peng Wang\textsuperscript{\rm 1},
        Enneng Yang\textsuperscript{\rm 1},
        Jiafei Lyu\textsuperscript{\rm 3},
        Xiaochun Cao\textsuperscript{\rm 1},
        Dacheng Tao\textsuperscript{\rm 2},
        Li Shen\textsuperscript{\rm 1}
        
\thanks{\textsuperscript{\rm 1}Guodong Zheng, Peng Wang, Enneng Yang, Xiaochun Cao and Li Shen are with the School of Cyber Science and Technology, Shenzhen Campus of Sun Yat-sen University, Shenzhen, China.
(E-mail: zhenggd3@mail.sysu.edu.cn; wangp389@mail.sysu.edu.cn; ennengyang@gmail.com; caoxiaochun@mail.sysu.edu.cn; mathshenli@gmail.com).}
\thanks{\textsuperscript{\rm 2}Yan Sun and Dacheng Tao are with the College of Computing and Data Science and the Generative AI Lab at Nanyang Technological University, Singapore. (Email: sun.yan.95@outlook.com; dacheng.tao@gmail.com))}%
\thanks{\textsuperscript{\rm 3}Jiafei Lyu is with the Tencent Hunyuan, Shenzhen, China. (Email: dmksjfl@gmail.com}}%

\markboth{Journal of \LaTeX\ Class Files,~Vol.~X, No.~X, XXX~XXX}%
{Shell \MakeLowercase{\textit{et al.}}: A Sample Article Using IEEEtran.cls for IEEE Journals}

\maketitle

\begin{abstract}
Understanding catastrophic forgetting, in which learning new tasks degrades performance on previously learned tasks, remains a fundamental challenge in continual learning (CL). Existing theoretical analyses often rely on specific model or task structures and offer limited insight into how iterative optimization affects forgetting. To address this gap, we develop an algorithmic stability framework for expected population forgetting under SGD without prescribing a particular model architecture. Specifically, we decompose forgetting into intra-task and cross-task generalization gaps and empirical forgetting. To bound the generalization gaps, we track how a sample perturbation affects updates within a task and propagates through subsequent tasks. We consider constant, intra-task decaying, and continuously decaying learning-rate schedules. Comparing the two decaying schedules allows us to characterize how restarting learning-rate decay at task boundaries affects the propagation of historical sample perturbations. For empirical forgetting, we derive bounds for all three schedules by retaining cumulative gradient alignment along the optimization trajectory, without requiring convexity. Combining these results yields upper bounds on expected forgetting that capture the effects of training set size, the number of updates per task, the total number of tasks, learning-rate schedules, and cumulative gradient alignment. Finally, we conduct experiments on Split CIFAR-100 and Split MiniImageNet to examine the empirical implications of our theoretical findings.
\end{abstract}


\section{Introduction}
% 导言区加入：\usepackage{tabularx}
\begin{table*}[t]
\begin{threeparttable}
\centering
\small
\caption{Results for the IGG, CGG, and EF under convex and non-convex loss functions with three learning-rate (LR) schedules: constant (C), intra-task decaying (ID), and continuously decaying (CD). Here, $T$ is the number of tasks and $i\in[T-1]$ is the index of a previous task. $N$ and $n$ denote the number of iterations and training samples per task, respectively. $L$ and $\beta$ are the Lipschitz and smoothness constants. We define $p=\frac{c\beta}{c\beta+1}$, where $c$ is the constant factor in the learning rate. The term $\mathcal{G}_{T,i}$ is cumulative gradient alignment defined in Equation~\eqref{Eqn: EF gradient inner product}. For the non-convex ID setting, the IGG and CGG upper bounds reported in the table apply when $2cL^2<N$ and $2cL^2(eN)^{c\beta(T-i)}<N$, respectively. When a condition is not satisfied, the corresponding upper bound is $N/n$. The full assumptions and detailed piecewise results are provided in Section \ref{sec:theoretical-analysis}.}
\label{tab:results}
\begin{tabularx}{\textwidth}{@{} l l >{\centering\arraybackslash}X >{\centering\arraybackslash}X >{\centering\arraybackslash}c @{}} 
\toprule
\textbf{Case} & \textbf{LR} & \textbf{IGG $\epsilon_{i,i}$} & \textbf{CGG $\epsilon_{T,i}$} & \textbf{EF $\tilde{\epsilon}_{T,i}$} \\
\midrule
\multirow{5}{*}{\rotatebox{90}{\textbf{Convex}}}
    & C  & $\frac{2cL^2N}{n}$ & $\frac{2cL^2N}{n}$ & $\scriptstyle -\mathcal{G}_{T,i}+\frac{\beta c^2L^2}{2}N(T-i)$ \\
    \noalign{\vskip 1mm} 
    \cmidrule{2-5}
    & ID & $\frac{2cL^2\ln(eN)}{n}$ & $\frac{2cL^2\ln(eN)}{n}$ & $\scriptstyle -\mathcal{G}_{T,i}+\frac{\beta c^2L^2}{2}(2-1/N)(T-i)$ \\
    \noalign{\vskip 1mm}
    \cmidrule{2-5}
    & CD & $\frac{2cL^2}{n}\ln\!\left(1+\frac{1}{i-1}\right)$& $\frac{2cL^2}{n}\ln\!\left(1+\frac{1}{i-1}\right)$ & $\scriptstyle -\mathcal{G}_{T,i}+\frac{\beta c^2L^2}{2}(1/(iN)-1/(TN))$  \\
\midrule
\multirow{3}{*}{\rotatebox{90}{\textbf{Non-convex}}}
    & ID & $\begin{gathered}\frac{(2cL^2)^{1-p}N^p}{pn}-\frac{2L^2}{\beta n}\end{gathered}$
         & $e^{p(T-i)}\mathcal{K}_{N^{T-i+1}}-\frac{2L^2(eN)^{c\beta(T-i)}}{\beta n}$
         & $\scriptstyle -\mathcal{G}_{T,i}+\frac{\beta c^2L^2}{2}(2-1/N)(T-i)$ \\
    \cmidrule{2-5}
    & \multirow{2}{*}{CD\tnote{a}}
        & $\mathcal{C}_{1}^1:\ \mathcal{K}_{iN}-\frac{(i-1)N}{n}-\frac{2L^2}{\beta n}$
        & $\mathcal{C}_{2}^1:\ \mathcal{K}_{TN}-\frac{(i-1)N}{n}-\frac{2L^2}{\beta n}(T/i)^{c\beta}$
        & \multirow{2}{*}{$\scriptstyle -\mathcal{G}_{T,i}+\frac{\beta c^2L^2}{2}(1/(iN)-1/(TN))$} \\
        \cmidrule{3-4}
    &
        & $\mathcal{C}_{1}^2:\ \frac{2L^2}{\beta n}\bigl[(\frac{i}{i-1})^{c\beta}-1\bigr]$
        & $\mathcal{C}_{2}^2:\ \frac{2L^2}{\beta n}\bigl[(\frac{T}{i-1})^{c\beta}-(\frac{T}{i})^{c\beta}\bigr]$
        & \\
\bottomrule
\end{tabularx}
\begin{tablenotes}
    \item [a] Let $\mathcal{C}_1=(2cL^2)^{1-p}(iN)^p-(i-1)N$ and $\mathcal{C}_2=(2cL^2)^{1-p}(TN)^p-(i-1)N$. For $j=1,2$, $\mathcal{C}_j^1$ denotes $0<\mathcal{C}_j<N$, and $\mathcal{C}_j^2$ denotes $\mathcal{C}_j\leq0$. When $\mathcal{C}_j\geq N$, the corresponding bound is $N/n$.
\end{tablenotes}
\end{threeparttable}
\vspace{-1em}
\end{table*}
Continual learning (CL) aims to learn from a sequence of non-stationary data streams and tasks over time. However, this classical training paradigm remains vulnerable to a fundamental challenge, known as \textit{catastrophic forgetting} \citep{mccloskey1989catastrophic,mcclelland1995there}, in which the model exhibits a substantial degradation in performance on previously learned tasks as it adapts to newly introduced tasks.

To mitigate catastrophic forgetting, several empirical methods have been proposed and have achieved promising results~\citep{Forgetting_Survey_2024}. However, the theoretical understanding remains much less developed. Existing analyses have largely focused on the \textit{linear regime}~\citep{evron2022catastrophic,lin2023theory,ding2024understanding,zhao2024statistical,banayeeanzade2025theoretical} or the \textit{neural tangent kernel (NTK) regime}~\citep{doan2021theoretical,karakida2022learning,zhengunderstanding}. While linear models yield explicit forgetting bounds, the resulting guarantees depend on the assumed linear structure. NTK-based analyses extend the study to neural networks, but typically rely on infinite-width limits. Moreover, most of these studies analyze forgetting through closed-form solutions to the optimization problem, abstracting away the fine-grained dynamics of iterative updates and yielding coarse bounds. Recently, \citet{zhengunderstanding} analyze \textit{forgetting dynamics} through NTK gradient flow, which captures more detailed trajectories than conventional static analyses. \citet{pmlr-v306-taheri26a} characterize forgetting in two-layer quadratic neural networks in the kernel regime, deriving explicit empirical forgetting bounds for orthogonal XOR-cluster tasks and using algorithmic stability to control delayed generalization. These studies provide insights into how forgetting evolves during training, while their guarantees remain tied to specific kernel approximations, network architectures, or task structures. This motivates an analysis of  training dynamics, without prescribing a particular network architecture, that tracks both within-task updates and the propagation of their effects across tasks.

In this work, we develop an algorithmic stability analysis of expected forgetting in CL. We quantify forgetting as the expected increase in population risk on historical tasks and decompose it into intra-task generalization gaps (IGG), cross-task generalization gaps (CGG), and empirical forgetting (EF). To bound the generalization terms, we characterize stability under a sample perturbation in either the current task or a historical task. We follow SGD trajectories to track how the perturbation affects updates within that task and propagates through subsequent tasks. This analysis yields generalization bounds for constant, intra-task decaying, and continuously decaying learning rate schedules in the convex setting, and for the two decaying schedules in the non-convex setting. Under intra-task decay, the learning rate decreases from the same initial value within each task, whereas under continuous decay, it continues decreasing across task boundaries. Comparing these schedules allows us to characterize how resetting the learning rate at task boundaries affects the propagation of historical sample perturbations. For EF, we analyze changes in empirical risk on historical tasks along the optimization trajectory, retaining the cumulative inner products between gradients of historical tasks and those of the current task. The EF analysis covers all three learning rate schedules without requiring convexity. Combining these results yields bounds on expected population forgetting that explicitly account for both generalization gaps and empirical risk changes.

As summarized in Table~\ref{tab:results}, the bounds reveal distinct effects of learning-rate scheduling on generalization and empirical forgetting. For convex losses with step sizes at most $2/\beta$, the IGG and CGG bounds scale as $\mathcal{O}(N/n)$ under constant learning rates and $\mathcal{O}(\log(eN)/n)$ under intra-task decay, while continuous decay gives bounds independent of $N$ for historical tasks $i>1$, with the other parameters fixed. For non-convex losses and $i>1$, the   IGG and CGG bounds under continuous decay are no larger than those under intra-task decay at the same parameter values. In particular, for fixed $c,\beta,L,T,i$ with $i>1$, the CGG bound under continuous decay is $\mathcal{O}(1/n)$ uniformly in $N$. Thus, increasing the local training length does not impose an additional growth restriction for this bound to vanish as $n$ increases.  These results show how carrying learning-rate decay across task boundaries changes the dependence of the generalization bounds on local training length. The EF bound further captures the joint contribution of cumulative gradient alignment and a second-order term controlled by squared learning rates. With the other parameters fixed, this second-order term scales as $\mathcal{O}(N)$, $\mathcal{O}(1)$, and $\mathcal{O}(1/N)$ under constant, intra-task decaying, and continuously decaying learning rates, respectively. Combining the generalization and EF bounds characterizes how perturbation propagation and trajectory dependent gradient interactions jointly enter the population forgetting bound. Finally, we conduct experiments on Split CIFAR-100 and Split MiniImageNet to examine the empirical implications of our theoretical findings.

Our main contributions can be summarized as follows:
\begin{itemize}
    \item We derive explicit IGG and CGG bounds for continual SGD by tracking sample perturbations within a task and through subsequent tasks. These bounds characterize the effects of sample size, local training length, the total number of tasks, and learning-rate schedules.
    \item We derive an EF bound in terms of cumulative gradient alignment and an explicit second-order term involving squared learning rates. The bound reveals how positive cumulative gradient alignment contributes to controlling empirical forgetting.
    \item We derive upper bounds on expected forgetting through 
    combining the IGG, CGG, and EF bounds, linking  generalization with changes in empirical risk along the continual learning trajectory.
\end{itemize}

\section{Related Work}
\textbf{Empirical studies in CL.}
Continual learning has made substantial progress in empirical methods for mitigating catastrophic forgetting. We divide these methods into three main families: (1) replay-based methods, which either store data from previous tasks and replay it during training on the current task \citep{chaudhry2018efficient,riemer2018learning,buzzega2020dark}, or store gradient or feature information from prior tasks and constrain the current-task gradient to be orthogonal to the stored information \citep{farajtabar2020orthogonal,sahagradient,lin2022trgp,yang2023data,yang2025revisiting_TPAMI_2025}; (2) regularization-based methods, which aim to restrain updates of parameters that are important for previous tasks \citep{kirkpatrick2017overcoming,zenke2017continual}, or maintain consistency of intermediate features by adding a regularization term to the loss \citep{li2017learning}; and (3) architecture-based methods, which dynamically expand or adapt the model architecture to preserve knowledge from previous tasks \citep{rusu2016progressive,yoon2018lifelong,wang2022sparcl}.

\textbf{Theoretical analysis of CL.}
Recent theoretical studies on CL have mainly focused on catastrophic forgetting and generalization. Broadly speaking, the existing literature can be divided into two main lines of work: linear model analysis and NTK-based analysis. Under the linear regime, \citet{evron2022catastrophic,lin2023theory} investigate forgetting and generalization under the assumptions of overparameterized linear models. \citet{ding2024understanding} study forgetting in both overparameterized and underparameterized regimes while relaxing the data assumptions to fourth-moment conditions. Moreover, several works use linear models to analyze concrete methods for mitigating forgetting, including replay-based methods \citep{banayeeanzade2025theoretical,zheng2025towards,deng2025unlocking}, regularization-based methods \citep{li2023fixed,zhao2024statistical,li2026memory,karpel2026optimal}, and mixture-of-experts methods \citep{li2025theory}. 
The other main line of theoretical work relies on the NTK regime. Specifically, \citet{bennani2020generalisation} establish generalization error bounds for CL based on the NTK and Rademacher complexity. \citet{doan2021theoretical} analyze forgetting by introducing the NTK overlap matrix as a task similarity metric. \citet{karakida2022learning} study forgetting and generalization error through learning-curve analysis between two tasks. \citet{zhengunderstanding} further characterize the training-time dynamics of forgetting and generalization in the general CL setting. \citet{pmlr-v306-taheri26a} analyze forgetting in two-layer quadratic neural networks in the kernel regime, deriving explicit empirical forgetting bounds for orthogonal XOR-cluster tasks and using algorithmic stability to control delayed generalization. \citet{liu2026parameter} derive generalization bounds for parameter-efficient fine-tuning in CL. However, these theoretical analyses have several limitations: (1) they rely on restrictive model assumptions, such as linear models, infinite-width models, or infinitesimal learning rates, which make the results less general; and (2) the upper bounds derived in the NTK setting often involve complicated kernel matrices, making them difficult to interpret and generalize. In contrast, we analyze forgetting through a fine-grained characterization of the training trajectory without restrictive model assumptions. Moreover, our bounds depend only on a few readily interpretable quantities, such as the training set size, local training length, learning-rate schedule, and gradient inner products between tasks.

The work most closely related to ours is that of \citet{pmlr-v306-taheri26a}, who analyze forgetting using algorithmic stability. Our analysis differs in three aspects: (1) Their theoretical results are developed for two-layer neural networks with quadratic activations and fixed output-layer weights, whereas our analysis does not prescribe a particular network architecture. (2) Their explicit empirical forgetting guarantees rely on orthogonal XOR-cluster tasks and the specified network model. Our  EF bound requires only that each sample loss be Lipschitz continuous and smooth with respect to the model parameters, without imposing these structural assumptions. Their delayed generalization analysis also uses Lipschitz and smoothness assumptions on the loss. (3) Their main forgetting and generalization theorems adopt constant learning rates, whereas we examine constant, intra-task decaying, and continuously decaying learning rates.

%In addition, \citet{wan2025understanding,dingprovable} leverage feature learning to analyze replay-based methods.

\section{Preliminary}
\subsection{Continual Learning Setting}
We consider the standard CL setting with $T\geq2$ sequential training tasks. For any positive integer $n$, let $[n]:=\{1,\dots,n\}$. For each task $\tau\in[T]$, let $\mathcal{D}_\tau$ denote the data distribution. To simplify the analysis, we assume that each task dataset contains $n$ samples, and write $S_\tau=(z_{\tau,1},\dots,z_{\tau,n})$ for the corresponding training dataset drawn i.i.d.\ from $\mathcal{D}_\tau$. The datasets are independent across tasks. We denote the task sequence by $S_{1:T}:=\{S_1,\ldots,S_T\}$, which collects the datasets from tasks $1$ to $T$. Let $f$ denote a differentiable loss function on the parameter space $\Omega=\mathbb{R}^d$. We adopt stochastic gradient descent (SGD) as our update rule, which can be formulated as
\[
w_{\tau,t+1}=w_{\tau,t}-\eta_{\tau,t}\nabla_w f(w_{\tau,t};z_{\tau,i_t}),
\]
where $\eta_{\tau,t}\geq0$ is the learning rate and $i_t$ is sampled uniformly from $[n]$ with replacement, independently across all updates and independently of the datasets and initialization. Here, $w_{\tau,t}$ denotes the model parameters at step $t$ for task $\tau$. Each task is trained for $N$ steps indexed by $t=0,\ldots,N-1$. We write $w_\tau=w_{\tau,N}$ for the model parameters obtained after training on task $\tau$, with $w_{\tau,0}=w_{\tau-1}$ and an initialization $w_0$ independent of all task datasets.

Furthermore, we define the population risk of $w_\tau$ on $\mathcal{D}_i$ as $R_{\mathcal{D}_i}(w_\tau)$ and the empirical risk of $w_\tau$ on $S_i$ as $R_{S_i}(w_\tau)$, given by
\[
\begin{aligned}
R_{\mathcal{D}_i}(w_\tau)&:=\mathbb{E}_{z\sim\mathcal{D}_i} f(w_\tau;z),\\
R_{S_i}(w_\tau)&:=\frac{1}{n}\sum_{k=1}^{n} f(w_\tau;z_{i,k}).
\end{aligned}
\]

We next adopt the definition of forgetting from \citet{zhengunderstanding,pmlr-v306-taheri26a}, which characterizes the model’s performance over the underlying data distribution. The formal definition is given in Definition~\ref{Def: forgetting}.

\begin{definition}[\textbf{Forgetting in expectation}]\label{Def: forgetting}
For any $i\in[T]$, let $w_i=\mathcal{A}(S_{1:i})$ denote the model obtained by the randomized algorithm $\mathcal{A}$ after training on the task sequence $S_{1:i}$. We define forgetting in expectation as
\begin{equation}\label{Eqn: forgetting}
    F_T = \frac{1}{T-1}\sum_{i=1}^{T-1}\mathbb{E}_{S_{1:T},\mathcal{A}}\big[R_{\mathcal{D}_i}(w_T) - R_{\mathcal{D}_i}(w_i)\big].
\end{equation}
\end{definition}

\textbf{Discussion.}
Definition~\ref{Def: forgetting} generalizes backward transfer (BWT)~\citep{lopez2017gradient}, a conventional metric for evaluating forgetting, in two respects. First, BWT measures forgetting on a test dataset, whereas $F_T$ measures forgetting over the underlying data distribution. Second, $F_T$ takes an expectation over the randomness in both the training dataset and the learning algorithm. These extensions yield a theoretical measure of expected population forgetting, with smaller values of $F_T$ indicating less forgetting.

\subsection{Algorithmic Stability} \label{Sec: algorithmic-stability}
In the non-CL setting, standard algorithmic stability typically concerns perturbations in the current training data only. In CL, by contrast, the model is exposed to a sequence of tasks, so perturbations at different tasks can propagate differently through the subsequent training trajectory. To reflect this sequential structure, we propose intra-task uniform stability for perturbations on the current task and cross-task uniform stability for perturbations on a historical task.

\begin{definition}[\textbf{Intra-task uniform stability}]\label{Def: uniform-stability-ii}
For any $i\in[T]$, a randomized algorithm $\mathcal{A}$ is said to be $\epsilon_{i,i}$-uniformly stable on task $i$ after training on $i$ tasks if, for any datasets $S_{1:i}, S_{1:i}' \in \mathcal{Z}^{i\times n}$ that differ in at most one example in task $i$, we have
\begin{equation}
    \sup_{z}\, \mathbb{E}_{\mathcal{A}}\!\left[f(\mathcal{A}(S_{1:i});z)-f(\mathcal{A}(S_{1:i}');z)\right]\leq \epsilon_{i,i}.
\end{equation}
\end{definition}

\begin{definition}[\textbf{Cross-task uniform stability}]\label{Def: uniform-stability-Ti}
For any $i\in[T-1]$, a randomized algorithm $\mathcal{A}$ is said to be $\epsilon_{T,i}$-uniformly stable on task $i$ after training on $T$ tasks if, for any datasets $S_{1:T}, S_{1:T}' \in \mathcal{Z}^{T\times n}$ that differ in at most one example in task $i$, we have
\begin{equation}
    \sup_{z}\, \mathbb{E}_{\mathcal{A}}\!\left[f(\mathcal{A}(S_{1:T});z)-f(\mathcal{A}(S_{1:T}');z)\right]\leq \epsilon_{T,i}.
\end{equation}
\end{definition}

In the following analysis, $\epsilon_{i,i}$ and $\epsilon_{T,i}$ denote the infima of the nonnegative constants satisfying Definitions~\ref{Def: uniform-stability-ii} and~\ref{Def: uniform-stability-Ti}, respectively.

\begin{remark}
    As shown in Definition~\ref{Def: uniform-stability-ii}, the stability measure $\epsilon_{i,i}$-uniform stability characterizes the algorithm's stability to perturbations in the data of the current task during CL. In contrast, Definition~\ref{Def: uniform-stability-Ti} extends this notion to perturbations in a historical task, thereby capturing how the effect of a local change can propagate through the subsequent training trajectory. This reflects the cumulative of stability in CL, where the influence of earlier tasks may persist and affect the final model after training on all tasks.
\end{remark}

In this work, we follow \citet{hardt2016train,pmlr-v235-wang24bl,chen2025error,liu2025understanding,zhu2026stability} to adopt the standard assumptions used in algorithmic stability analysis:
\begin{assumption}[\textbf{$L$-Lipschitz}]\label{assump: Lipschitz}
For every $z\in\mathcal{Z}$, the function $f(\cdot;z):\Omega\rightarrow\mathbb{R}$ is $L$-Lipschitz, that is, for all $u,v\in\Omega$,
\[
|f(u;z)-f(v;z)|\leq L\|u-v\|.
\]
\end{assumption}

\begin{assumption}[\textbf{$\beta$-smooth}]\label{assump: smooth}
For every $z\in\mathcal{Z}$, the function $f(\cdot;z):\Omega\rightarrow\mathbb{R}$ is $\beta$-smooth, that is, for all $u,v\in\Omega$,
\[
\|\nabla f(u;z)-\nabla f(v;z)\|\leq \beta\|u-v\|.
\]
\end{assumption}

\begin{assumption}[\textbf{Convex}]\label{assump: convex}
For every $z\in\mathcal{Z}$, the function $f(\cdot;z):\Omega\rightarrow\mathbb{R}$ is convex, that is, for all $u,v\in\Omega$,
\[
f(u;z)\geq f(v;z)+\langle \nabla f(v;z),u-v\rangle.
\]
\end{assumption}

% \begin{assumption}[\textbf{Task-similarity}]\label{assump: task_similarity}
% For any two data distributions $\mathcal{D}_i$ and $\mathcal{D}_j$ associated with tasks 
% $i,j\in[T]$, the squared Wasserstein distance between them is uniformly bounded. Specifically, there exists a constant $\sigma^2\ge 0$ such that
% \[
% W_2^2(\mathcal{D}_i,\mathcal{D}_j)\le \frac{1}{3}\sigma^2,
% \qquad \forall\, i,j\in[T].
% \]
% \end{assumption}

\section{Theoretical Analysis}\label{sec:theoretical-analysis}
In this section, we provide a theoretical analysis of catastrophic forgetting. To comprehensively analyze the forgetting measure $F_T$ in Equation~\eqref{Eqn: forgetting}, it is important to derive a fine-grained decomposition that allows us to understand forgetting from multiple perspectives:

\begin{equation}\label{Eqn: forgetting decomposed}
\begin{aligned}
F_T = \frac{1}{T-1}\sum_{i=1}^{T-1}\Bigl[&
\underbrace{\mathbb{E}_{S_{1:T},\mathcal{A}}[R_{\mathcal{D}_i}(w_T)-R_{S_i}(w_T)]}_{\text{Cross-Task Generalization Gap}}\\
+&\underbrace{\mathbb{E}_{S_{1:T},\mathcal{A}}[R_{S_i}(w_T)-R_{S_i}(w_i)]}_{\text{Empirical Forgetting}}\\
+&\underbrace{\mathbb{E}_{S_{1:T},\mathcal{A}}[R_{S_i}(w_i)-R_{\mathcal{D}_i}(w_i)]}_{\text{Intra-Task Generalization Gap}}\Bigr].
\end{aligned}
\end{equation}
Based on the decomposition in Equation~\eqref{Eqn: forgetting decomposed}, we analyze the forgetting measure $F_T$ from the following three aspects.

(1) \textit{Intra-Task Generalization Gap} (Section~\ref{Sec: IGG}). This term captures the generalization behavior of the trained model on the current task. A smaller absolute intra-task generalization gap indicates closer agreement between the empirical risk and the population risk on the current task.

(2) \textit{Cross-Task Generalization Gap} (Section~\ref{Sec: CGG}). This term reflects the generalization behavior of the model trained on later tasks when evaluated on earlier task distributions. A smaller absolute cross-task generalization gap indicates closer agreement between the empirical risk and the population risk on previous tasks.
 
(3) \textit{Empirical Forgetting} (Section~\ref{Sec: EF}). This term captures how much performance on the training set of task $i$ deteriorates after learning $T$ tasks. A smaller empirical forgetting indicates that the model forgets less of the knowledge from the training set of a previous task.

In what follows, we derive upper bounds for these three terms under convex  and non-convex loss functions, respectively, to provide a comprehensive understanding of catastrophic forgetting in CL.

\subsection{Intra-Task Generalization Gap}\label{Sec: IGG}
The intra-task generalization gap (IGG) describes the generalization performance on the current task. We next present an important theorem showing that intra-task uniform stability implies an upper bound on the intra-task generalization gap in expectation. The result is stated in Theorem~\ref{Thm: IGG}, and the corresponding proof is given in Appendix~\ref{Sec: Proof of IGG}.

\begin{theorem}\label{Thm: IGG} For any $i\in[T]$, if $\mathcal{A}$ is $\epsilon_{i,i}$-uniformly stable, then
    \begin{equation*}
        |\mathbb{E}_{S_{1:i},\mathcal{A}}[R_{\mathcal{D}_i}(\mathcal{A}(S_{1:i}))-R_{S_i}(\mathcal{A}(S_{1:i}))]|\leq\epsilon_{i,i}.
    \end{equation*}
\end{theorem}

Theorem~\ref{Thm: IGG} shows that an intra-task uniformly stable algorithm will have a small IGG. This suggests that we can control the uniform stability itself, rather than directly bounding the generalization gap. We then provide the corresponding  upper bound under both convex and non-convex loss functions. 

First, we establish an upper bound for the case where the loss function is \textit{convex} in Theorem~\ref{Thm: IGG bound convex}, and the proof of Theorem~\ref{Thm: IGG bound convex} is given in Appendix~\ref{Sec: proof of IGG bound convex}.

\begin{theorem}\label{Thm: IGG bound convex}
Assume that the loss function $f(\cdot;z)$ satisfies Assumptions~\ref{assump: Lipschitz}, \ref{assump: smooth}, and \ref{assump: convex}. Suppose that SGD is run for $N$ steps on each task $\tau\in[i]$ with step sizes $\eta_{\tau,t}\leq \frac{2}{\beta}$. Then,
\begin{equation}\label{Eqn: bound e_ii convex}
    \epsilon_{i,i}\leq \frac{2L^2}{n}\sum_{t=0}^{N-1} \eta_{i,t}.
\end{equation}
\end{theorem}

\begin{remark}
Theorem~\ref{Thm: IGG bound convex} implies that, under convex loss functions, the IGG can be bounded in terms of the cumulative learning rate within task $i$. This bound is consistent with the classical non-CL result of \citet{hardt2016train}. Hence, in the convex setting, the bound on $\epsilon_{i,i}$ is independent of previous tasks and depends only on the optimization process within the current task. Moreover, smaller learning rates and larger numbers of training samples can further reduce this upper bound.
\end{remark}
Moreover, for different learning rate schedules, including the constant learning rate $\eta_{\tau,t}=c$, the intra-task decaying learning rate $\eta_{\tau,t}=\frac{c}{t+1}$, and the continuously decaying learning rate $\eta_{\tau,t}=\frac{c}{(\tau-1)N+t+1}$, Equation~\eqref{Eqn: bound e_ii convex} can be further written as
\begin{equation*}
    \epsilon_{i,i}\leq
\begin{cases}
\frac{2cL^2N}{n}, 
& \text{if } \eta_{\tau,t}=c, \\[2ex]
\frac{2cL^2\ln(eN)}{n}, 
& \text{if } \eta_{\tau,t}=\frac{c}{t+1}, \\[2ex]
\frac{2cL^2}{n}\ln\!\left(1+\frac{1}{i-1}\right), 
& \text{if } \eta_{\tau,t}=\frac{c}{(\tau-1)N+t+1},i>1.
\end{cases}
\end{equation*}
\begin{remark}\label{Rem: IGG convex diff. lr}
The upper bound on $\epsilon_{i,i}$ varies across learning-rate schedules. Under a constant learning rate, the bound grows linearly with the number of SGD steps $N$ on the current task. Under an intra-task decaying learning rate, this growth is reduced to $\mathcal{O}(\ln(eN))$. With other parameters fixed, reducing the number of training steps yields a smaller upper bound under both schedules, providing a theoretical motivation for early stopping. Under the continuously decaying learning-rate schedule, the bound coincides with that under intra-task decay when $i=1$. For $i>1$, the displayed upper bound is independent of $N$ and is instead controlled by the task index $i$.
\end{remark}

Before presenting the IGG bounds for non-convex losses, we introduce the following notation for a recurring term in the subsequent analysis:
\[
\mathcal{K}_N:=\frac{c\beta+1}{nc\beta}(2cL^2)^{\frac{1}{c\beta+1}}N^{\frac{c\beta}{c\beta+1}}.
\]
The subscript specifies the number of iterations, with $c$, $\beta$, $L$, and $n$ held fixed. In particular, for any positive integer $i$, 
\[
\mathcal{K}_{iN}
=\frac{c\beta+1}{nc\beta}(2cL^2)^{\frac{1}{c\beta+1}}(iN)^{\frac{c\beta}{c\beta+1}}.
\]

Then, we derive an upper bound for the more challenging \textit{non-convex} case in Theorem~\ref{Thm: IGG bound non-convex}, and the proof of Theorem~\ref{Thm: IGG bound non-convex} is given in Appendix~\ref{Sec: proof of IGG bound non-convex}. 

\begin{theorem}\label{Thm: IGG bound non-convex}
Assume that the loss function $f(\cdot;z)\in[0,1]$ satisfies Assumptions~\ref{assump: Lipschitz} and \ref{assump: smooth}, with $n\geq2$ and $c,\beta>0$. Suppose that SGD is run for $N$ steps on each task $\tau\in[i]$ with learning rates $\eta_{\tau,t}$. Then,

(a) under intra-task decaying learning rates $\eta_{\tau,t}=\frac{c}{t+1}$:
    \[
    \epsilon_{i,i}\leq 
    \begin{cases} 
          \mathcal{K}_N-\frac{2L^2}{\beta n},&\text{if }2cL^2<N,\\[1ex]
          \frac{N}{n},&\text{if }2cL^2\geq N.
    \end{cases}
    \]
    
(b) under continuously decaying learning rates $\eta_{\tau,t}=\frac{c}{(\tau-1) N+t+1}$ and $i>1$:
    \[
    \epsilon_{i,i}\leq 
    \begin{cases} 
          \frac{2L^2}{\beta n}\left[\left(\frac{i}{i-1}\right)^{c\beta}-1\right],&\text{if }2cL^2\leq \mathcal{C}_1^{\mathrm{CD}},\\[1ex]
          \mathcal{K}_{iN}-\frac{(i-1)N}{n}-\frac{2L^2}{\beta n},
          &\text{if }\mathcal{C}_1^{\mathrm{CD}}<2cL^2<iN,\\[1ex]
          \frac{N}{n},&\text{if }2cL^2\geq iN,
    \end{cases}
    \]
    where 
    \[
    \begin{aligned}
        \mathcal{C}_1^{\mathrm{CD}}&=(i-1)N\left(\frac{i-1}{i}\right)^{c\beta}.
    \end{aligned}
    \]
\end{theorem}

\begin{remark}\label{Rem: IGG_ID_non-convex}
Theorem~\ref{Thm: IGG bound non-convex} provides an upper bound for the IGG without assuming convexity of the loss function. Let $p=c\beta/(1+c\beta)$. Under intra-task decaying learning rates, when $2cL^2<N$, the bound scales as $\mathcal{O}(N^p/n)$ for fixed $c,\beta,L>0$, consistent with the non-convex SGD stability rate established by \citet{hardt2016train} in the single-task setting. As $n\to\infty$, the condition $N=o(n^{1/p})$ is sufficient for a vanishing IGG. Thus, the number of training steps can increase with the sample size, provided that it grows more slowly than $n^{1/p}$. When $2cL^2\geq N$, corresponding to a larger initial learning rate or Lipschitz constant relative to the training length, the bound becomes $N/n$. In both cases, increasing the sample size or reducing the number of training steps decreases the displayed upper bound, with the other parameters fixed.
\end{remark}

\begin{remark}\label{Rem: IGG_CD_non-convex}
For $i>1$, continuously decaying learning rates provide stronger control of the IGG bound as the training length increases. When $2cL^2\leq\mathcal{C}_1^{\mathrm{CD}}$, the bound is independent of $N$ and decreases with the task index $i$. Since $i\geq2$, it is at most $\frac{2L^2}{\beta n}(2^{c\beta}-1)$. In the intermediate case $\mathcal{C}_1^{\mathrm{CD}}<2cL^2<iN$, the bound increases with $N$ but remains below $\frac{2L^2}{\beta n}[(i/(i-1))^{c\beta}-1]$. Thus, the dependence on $N$ in this case does not lead to unbounded growth. For the same $c,\beta,L,n,N$, the displayed bound under continuous decay is no larger than that under intra-task decay, and the improvement is strict when $0<2cL^2<iN$. When $2cL^2\geq iN$, both schedules give the bound $N/n$, which remains nontrivial for $N<n$. Moreover, increasing the number of training samples and decreasing the initial learning rate can further reduce the upper bound under this learning-rate schedule.
\end{remark}

\subsection{Cross-Task  Generalization Gap}\label{Sec: CGG}
The cross-task generalization gap (CGG) captures the generalization performance on earlier tasks. Next, we present the relationship between cross-task uniform stability and the expected CGG, which reflects how perturbations on historical tasks propagate through training in the CL setting and affect CGG. The result is stated in Theorem~\ref{Thm: CGG}, and the proof is given in Appendix~\ref{Sec: Proof of CGG}.

\begin{theorem}\label{Thm: CGG}
For any $i\in[T-1]$, if $\mathcal{A}$ is $\epsilon_{T,i}$-uniformly stable, then
\begin{equation*}
    |\mathbb{E}_{S_{1:T},\mathcal{A}}[R_{\mathcal{D}_i}(\mathcal{A}(S_{1:T}))-R_{S_i}(\mathcal{A}(S_{1:T}))]|\leq\epsilon_{T,i}.
\end{equation*}
\end{theorem}

Theorem~\ref{Thm: CGG} indicates that small cross-task uniform stability leads to a small CGG.  We then analyze CGG under both convex and non-convex loss functions. 

First, we establish an upper bound for the case where the loss function is \textit{convex} in Theorem~\ref{Thm: CGG bound convex}, and the proof of Theorem~\ref{Thm: CGG bound convex} is given in Appendix~\ref{Sec: proof of CGG bound convex}.

\begin{theorem}\label{Thm: CGG bound convex}
Assume that the loss function $f(\cdot;z)$ satisfies Assumptions~\ref{assump: Lipschitz}, \ref{assump: smooth}, and \ref{assump: convex}. Suppose that SGD is run for $N$ steps on each task $\tau\in[T]$ with step sizes $\eta_{\tau,t}\leq \frac{2}{\beta}$. Then,
    \[\epsilon_{T,i}\leq\frac{2L^2}{n}\sum_{t=0}^{N-1} \eta_{i,t}.\]
\end{theorem}

\begin{remark}
Theorem~\ref{Thm: CGG bound convex} surprisingly shows that the bound for the CGG coincides with that for the IGG under convex loss functions. This indicates that, under convexity, the algorithmic stability of SGD is preserved throughout the entire sequence of tasks. Even as the model parameters are continuously updated using data from new tasks, the resulting bound on the expected generalization gap on historical tasks does not grow. This conclusion relies on the non-expansive property of convex and smooth loss functions, as established in Lemma~\ref{Lem: f property}. Therefore, the discussion of IGG bound in Remark~\ref{Rem: IGG convex diff. lr} remains valid for CGG as well.
\end{remark}

Then, we derive an upper bound for the more challenging \textit{non-convex} case in Theorem~\ref{Thm: CGG bound non-convex}, and the proof of Theorem~\ref{Thm: CGG bound non-convex} is given in Appendix~\ref{Sec: proof of CGG bound non-convex}.

\begin{theorem}\label{Thm: CGG bound non-convex}
Assume that the loss function $f(\cdot;z)\in[0,1]$ satisfies Assumptions~\ref{assump: Lipschitz} and \ref{assump: smooth}, with $n\geq2$ and $c,\beta>0$. Suppose that SGD is run for $N$ steps on each task $\tau\in[T]$ with learning rates $\eta_{\tau,t}$ for $t=0,\ldots,N-1$. Then,

(a) under intra-task decaying learning rates $\eta_{\tau,t}=\frac{c}{t+1}$:
    \[
     \epsilon_{T,i}\leq\begin{cases}
         \mathcal{B}_2^1, &\text{if }2cL^2(eN)^{c\beta(T-i)}<N,\\[2ex]
          \displaystyle\frac{N}{n}, &\text{if }2cL^2(eN)^{c\beta(T-i)}\geq N.
     \end{cases}
    \]

    where 
    \[
        \begin{aligned}
        \mathcal{B}_2^1=&\mathcal{K}_N(eN)^{\frac{c\beta(T-i)}{c\beta+1}}-\frac{2L^2(eN)^{c\beta(T-i)}}{\beta n}.
        \end{aligned}
    \]
(b) under continuously decaying learning rates $\eta_{\tau,t}=\frac{c}{(\tau-1)N+t+1}$ and $i>1$:
    {\small
    \[
    \epsilon_{T,i}
    \leq
    \begin{cases}
      \frac{2L^2}{\beta n}
      \left[\left(\frac{T}{i-1}\right)^{c\beta}-\left(\frac{T}{i}\right)^{c\beta}\right],&\text{if }2cL^2\leq\mathcal{C}_2^{\mathrm{CD}},\\[2ex]
      \mathcal{B}_2^2,
      &\text{if }\mathcal{C}_2^{\mathrm{CD}}<2cL^2<iN\left(\frac{i}{T}\right)^{c\beta},\\[2ex]
      \displaystyle\frac{N}{n}, &\text{if }2cL^2\geq iN\left(\frac{i}{T}\right)^{c\beta}.
    \end{cases}
    \]}
    where 
    {
    \[
    \begin{aligned}
        \mathcal{B}_2^2&=\mathcal{K}_{TN}-\frac{(i-1)N}{n}-\frac{2L^2}{\beta n}\left(\frac{T}{i}\right)^{c\beta},\\
        \mathcal{C}_2^{\mathrm{CD}}&=(i-1)N\left(\frac{i-1}{T}\right)^{c\beta}.
    \end{aligned}
    \]}
For $i=1$ under the continuously decaying schedule, 
\[
\epsilon_{T,1}\leq\begin{cases}
    \mathcal{K}_{TN}-\frac{2L^2T^{c\beta}}{\beta n},
    &\text{if }2cL^2T^{c\beta}<N,\\[2ex]
    \displaystyle\frac{N}{n},&\text{if }2cL^2T^{c\beta}\geq N.
\end{cases}
\]

\end{theorem}

\begin{remark}
Under intra-task decaying learning rates, the displayed CGG bound is nondecreasing in the number of subsequent tasks and is no smaller than the corresponding IGG bound. When $2cL^2(eN)^{c\beta(T-i)}<N$, a sufficient condition for the CGG to vanish as $n\to\infty$ is $N=o(n^{1/[p(T-i+1)]})$, with $c,\beta,L,T,i$ fixed. Compared with the sufficient condition $N=o(n^{1/p})$ for the IGG, this condition imposes a stronger restriction on the training length, reflecting the additional effect of perturbation propagation across subsequent tasks. The negative term in $\mathcal{B}_2^1$ partially offsets the positive term and therefore tightens the bound. Within this branch, the ratio of the magnitude of the negative term to the positive term increases with the number of subsequent tasks when the other parameters are fixed, but remains strictly smaller than $1/(c\beta+1)$. Thus, retaining the negative correction provides a greater relative reduction for tasks followed by more subsequent training. Increasing the sample size or reducing the number of training steps also decreases the displayed upper bound, with the other parameters fixed.
\end{remark}

\begin{remark}
Under continuous decay, the learning rate continues to decrease across task boundaries, whereas intra-task decay resets it to $c$ at the beginning of each task. The corresponding upper bounds on perturbation amplification across subsequent tasks are $(T/i)^{c\beta}$ and $(eN)^{c\beta(T-i)}$, respectively. Continuous decay thus gives a smaller amplification bound that is independent of $N$. For $i>1$, the displayed CGG bound under continuous decay is also no larger than that under intra-task decay with the same parameters. All three cases of this bound are bounded above by $\frac{2L^2}{\beta n}[(T/(i-1))^{c\beta}-(T/i)^{c\beta}]$, which coincides with the first-case expression when $2cL^2\leq(i-1)N((i-1)/T)^{c\beta}$. This common upper bound decreases as $i$ increases, with the other parameters fixed, and its independence from $N$ extends the comparison beyond any particular training length. Consequently, for fixed $c,\beta,L,T,i$ with $i>1$, the CGG bound is $\mathcal{O}(1/n)$ uniformly in $N$ and vanishes as $n\to\infty$ without restricting the growth of the training length.
\end{remark}

\subsection{Empirical Forgetting}\label{Sec: EF}
Empirical Forgetting (EF) quantifies the degradation in performance on the training set of a previous task caused by learning subsequent tasks. We denote $\tilde{\epsilon}_{T,i}=\mathbb{E}_{S_{1:T},\mathcal{A}}[R_{S_i}(w_T)-R_{S_i}(w_i)]$ as the expected change in the empirical risk on task $i$ between the completion of task $i$ and the completion of task $T$. To characterize the interactions between task $i$ and subsequent tasks during training, we define the \textit{cumulative gradient alignment} (CGA) as
\begin{equation}\label{Eqn: EF gradient inner product}
\begin{aligned}
\mathcal{G}_{T,i}:=&\sum_{\tau=i+1}^{T}\sum_{t=0}^{N-1}\eta_{\tau,t}
\mathbb{E}_{S_{1:T},\mathcal{A}}\!\bigl\langle
\nabla R_{S_i}(w_{\tau,t}),\nabla R_{S_\tau}(w_{\tau,t})\bigr\rangle.
\end{aligned}
\end{equation}
Here, $\mathcal{G}_{T,i}$ accumulates the expected inner products between the gradients of task $i$ and those of subsequent tasks, weighted by the corresponding learning rates. Positive and negative values indicate cumulative gradient agreement and interference, respectively. We then provide an upper bound on $\tilde{\epsilon}_{T,i}$ in Theorem~\ref{Thm: EF}, and the proof is given in Appendix~\ref{Sec: proof of EF}.

\begin{theorem}\label{Thm: EF}
    Assume that the loss function $f(\cdot;z)$ satisfies Assumptions~\ref{assump: Lipschitz} and~\ref{assump: smooth}. Suppose that SGD is run for $N$ steps on each task $\tau\in[T]$ with learning rates $\eta_{\tau,t}$. Then,\\
    (a) under constant learning rates $\eta_{\tau,t}=c$:
    \begin{align*}
        \tilde{\epsilon}_{T,i} \leq -\mathcal{G}_{T,i}+\frac{\beta c^2L^2}{2}N(T-i);
    \end{align*}
    (b) under intra-task decaying learning rates $\eta_{\tau,t}=\frac{c}{t+1}$:
    \begin{align*}
        \tilde{\epsilon}_{T,i} \leq -\mathcal{G}_{T,i}+\frac{\beta c^2L^2}{2}\left(2-\frac1N\right)(T-i);
    \end{align*}
    (c) under continuously decaying learning rates $\eta_{\tau,t}=\frac{c}{(\tau-1)N+t+1}$:
    \begin{align*}
        \tilde{\epsilon}_{T,i} \leq -\mathcal{G}_{T,i}+\frac{\beta c^2L^2}{2}\left(\frac1{iN}-\frac1{TN}\right).
    \end{align*}
\end{theorem}

\begin{remark}\label{rem:ef}
Theorem~\ref{Thm: EF} bounds empirical forgetting in terms of the cumulative inner products between the old-task and current-task gradients and the squared learning rates. A positive $\mathcal{G}_{T,i}$ reduces the upper bound, while a negative value increases it. For fixed $T$, $i$, $c$, $L$, and $\beta$, the second-order term grows linearly with $N$ under a constant learning rate. Under intra-task decay, it approaches $\beta c^2L^2(T-i)$ as $N$ increases and thus becomes only weakly dependent on $N$ for large $N$. Under continuous decay, the term decreases proportionally to $1/N$. These results hold for both convex and non-convex losses.
\end{remark}

As discussed in Remark~\ref{rem:ef}, a positive $\mathcal{G}_{T,i}$ reduces the upper bound on empirical forgetting, indicating that positive gradient alignment accumulated over subsequent updates can help preserve performance on task $i$. This observation motivates the following connection with existing CL methods.

\textbf{Relation to gradient alignment mechanisms.} GEM~\citep{lopez2017gradient} aims to mitigate negative BWT by projecting the current gradient onto the set of directions having nonnegative inner products with all previous-task memory gradients. Let $\mathcal{M}_i$ denote the episodic memory containing samples from task $i$. At step $t$ of task $\tau$, the projected gradient is obtained by solving
\begin{equation}
\begin{aligned}
\operatorname*{min}_{\widetilde{g}_{\tau,t}}\quad
&\frac12
\left\|\widetilde{g}_{\tau,t}-\nabla f(w_{\tau,t};z_{\tau,i_t})\right\|^2\\
\mathrm{s.t.}\quad
&\left\langle
\widetilde{g}_{\tau,t},\frac{1}{|\mathcal{M}_i|}
\sum_{z\in\mathcal{M}_i}\nabla f(w_{\tau,t};z)
\right\rangle\geq0,
\quad \forall i<\tau.
\end{aligned}
\end{equation}
The parameters are then updated as $w_{\tau,t+1}=w_{\tau,t}-\eta_{\tau,t}\widetilde{g}_{\tau,t}$. To reduce the computational cost of projection, A-GEM~\citep{chaudhry2018efficient} replaces these individual constraints with a single constraint based on a reference gradient estimated from samples drawn from the combined episodic memory. Both methods seek to reduce gradient interference with past data, which is consistent with the role of cumulative gradient alignment captured by $-\mathcal{G}_{T,i}$ in our EF bound.

\textbf{Observable CGA.}
To examine the role of gradient alignment in empirical forgetting, we accumulate the learning-rate-weighted inner products between historical evaluation-loss gradients and the actual training gradients. We estimate the historical gradients using measurement buffers sampled from previous tasks and denote the resulting cumulative alignment by $\widehat{\mathcal{G}}_{T,i}$.
Historical gradients are evaluated using the same loss computation as EF, while current gradients are those used for the training updates. The buffers are used exclusively for measurement. Details are provided in Appendix~\ref{Appendix:experiments-on-miniimagenet}.

\begin{figure}[htbp]
    \centering
    \begin{subfigure}[b]{0.485\textwidth}
        \centering
        \includegraphics[width=\linewidth]{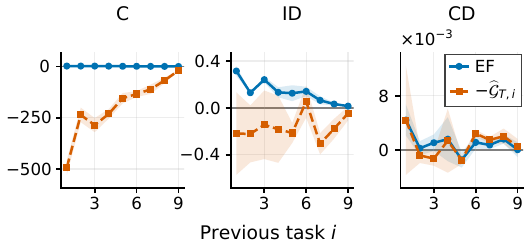}
    \end{subfigure}
    
    \caption{\textbf{EF and CGA under different learning-rate schedules on Split CIFAR-100.} C, ID, and CD denote constant, intra-task decaying, and continuously decaying learning rates, respectively. Curves show the means over three seeds, and shaded regions indicate one standard deviation.}
\label{fig:alignment}
\end{figure}

As shown in Fig.~\ref{fig:alignment}, the estimated CGA is large and positive under constant learning rates, particularly for earlier tasks. Under intra-task decay, its magnitude is substantially smaller, although EF and $-\widehat{\mathcal{G}}_{T,i}$ remain visibly separated. Under continuous decay, the two quantities are closer and exhibit similar fluctuations across historical tasks. We also report results on Split MiniImageNet and provide details of the CGA measurements in Appendix~\ref{Appendix:experiments-on-miniimagenet}.

\subsection{Forgetting in Expectation}
In this section, we summarize the discussions in Sections~\ref{Sec: IGG}, \ref{Sec: CGG}, and \ref{Sec: EF}. Specifically, the forgetting measure $F_T$ can be bounded by the following inequality:
\begin{equation}\label{Eqn:upp-FT}
    F_T\le\frac{1}{T-1}\sum_{i=1}^{T-1}\left(\epsilon_{T,i}+\epsilon_{i,i}+\tilde\epsilon_{T,i}\right).
\end{equation}
We next provide upper bounds on $F_T$ for convex and non-convex losses in Theorems~\ref{Thm: FT-convex} and~\ref{Thm: FT-non-convex}, respectively.

\begin{theorem}\label{Thm: FT-convex}
Assume that the loss function $f(\cdot;z)$ satisfies Assumptions~\ref{assump: Lipschitz}, \ref{assump: smooth}, and \ref{assump: convex}, with $\beta>0$. Suppose that SGD is run for $N$ steps on each task $\tau\in[T]$ with learning rates $\eta_{\tau,t}\leq2/\beta$. For $0<c\leq2/\beta$, the following bounds hold:\\
(a) under constant learning rates $\eta_{\tau,t}=c$:
\[
F_T\leq-\frac{\sum_{i=1}^{T-1}\mathcal{G}_{T,i}}{T-1}+\frac{\beta c^2L^2NT}{4}+\frac{4cL^2N}{n};
\]
(b) under intra-task decaying learning rates $\eta_{\tau,t}=\frac{c}{t+1}$:
\[
\begin{aligned}
F_T\leq{}&-\frac{\sum_{i=1}^{T-1}\mathcal{G}_{T,i}}{T-1}+\frac{\beta c^2L^2T}{4}\left(2-\frac1N\right)+\frac{4cL^2\ln(eN)}{n};
\end{aligned}
\]
(c) under continuously decaying learning rates $\eta_{\tau,t}=\frac{c}{(\tau-1)N+t+1}$:
\[
\begin{aligned}
F_T\leq{}&-\frac{\sum_{i=1}^{T-1}\mathcal{G}_{T,i}}{T-1}+\frac{\beta c^2L^2}{2N(T-1)}\left(\ln(T-1)+\frac1T\right)\\
&+\frac{4cL^2}{n(T-1)}\ln\big(eN(T-1)\big).
\end{aligned}
\]
\end{theorem}

\begin{theorem}\label{Thm: FT-non-convex}
Assume that the loss function $f(\cdot;z)\in[0,1]$ satisfies Assumptions~\ref{assump: Lipschitz} and \ref{assump: smooth}, with $n\geq2$ and $c,\beta,L>0$. Suppose that SGD is run for $N$ steps on each task $\tau\in[T]$ with learning rates $\eta_{\tau,t}$. Then,\\
(a) under intra-task decaying learning rates $\eta_{\tau,t}=\frac{c}{t+1}$, if $2cL^2(eN)^{c\beta(T-1)}<N$, then
\begin{equation}\label{Eqn: FT ID explicit}
\begin{aligned}
F_T\leq{}&-\frac{\sum_{i=1}^{T-1}\mathcal{G}_{T,i}}{T-1}
+\frac{\beta c^2L^2T}{4}\left(2-\frac1N\right)\\
&+\mathcal{K}_N\left[1+\frac{(eN)^{\frac{c\beta(T-1)}{c\beta+1}}-1}
{(T-1)\left(1-(eN)^{-\frac{c\beta}{c\beta+1}}\right)}\right]\\
&-\frac{2L^2}{\beta n}\left[1+
\frac{(eN)^{c\beta(T-1)}-1}{(T-1)\left(1-(eN)^{-c\beta}\right)}\right].
\end{aligned}
\end{equation}
(b) under continuously decaying learning rates $\eta_{\tau,t}=\frac{c}{(\tau-1)N+t+1}$:
\begin{equation}\label{Eqn: FT CD explicit}
\begin{aligned}
F_T\leq&-\frac{\sum_{i=1}^{T-1}\mathcal{G}_{T,i}}{T-1}+\frac{\beta c^2L^2}{2N(T-1)}\left(\ln(T-1)+\frac1T\right)\\
&+\frac{2L^2}{\beta n(T-1)}\frac{2^{c\beta}-1}{\ln2}\ln(T-1)+\frac1{(T-1)}\times\\
&\begin{cases}
\mathcal{K}_N\left(1+T^{\frac{c\beta}{c\beta+1}}\right)-\frac{2L^2}{n\beta}\left[1+\left(\frac{T}{T-1}\right)^{c\beta}\right],
&\text{if }\mathcal{C}_3^1,
\\[2ex]
\mathcal{K}_N-\frac{2L^2}{n\beta}+\frac{(T-1)N}{n},&\text{if }\mathcal{C}_3^2,\\[2ex]
\mathcal{K}_{TN}-\frac{2L^2}{n\beta}\left(\frac{T}{T-1}\right)^{c\beta}+\frac{N}{n},&\text{if }\mathcal{C}_3^3,
\\[2ex]
\frac{NT}{n},&\text{if }\mathcal{C}_3^4,
\end{cases}
\end{aligned}
\end{equation}
where
\[
\begin{aligned}
    &\mathcal{C}_3^1:2cL^2<N, \quad 2cL^2T^{c\beta}<N(T-1)^{c\beta+1},\\
    &\mathcal{C}_3^2:2cL^2<N, \quad 2cL^2T^{c\beta}\geq N(T-1)^{c\beta+1},\\
    &\mathcal{C}_3^3:2cL^2\geq N, \quad 2cL^2T^{c\beta}<N(T-1)^{c\beta+1},\\
    &\mathcal{C}_3^4:2cL^2\geq N,\quad 2cL^2T^{c\beta}\geq N(T-1)^{c\beta+1}.
\end{aligned}
\]
\end{theorem}

The proofs of Theorems~\ref{Thm: FT-convex} and~\ref{Thm: FT-non-convex} are given in Appendix~\ref{Sec: Forgetting}. To obtain explicit bounds on $F_T$, we evaluate the accumulated contributions of IGG, CGG, and EF across previous tasks. As shown in Equation~\eqref{Eqn:upp-FT}, the upper bound on forgetting can be reduced by controlling $\epsilon_{T,i}$, $\epsilon_{i,i}$, and $\tilde{\epsilon}_{T,i}$ across historical tasks $i\in[T-1]$. If the IGG and CGG bounds and the magnitude of EF all vanish, then $F_T$ also converges to zero. The strategies for controlling these terms have been discussed in Sections~\ref{Sec: IGG}--\ref{Sec: EF}. We summarize the implications of the bounds as follows: \\
(1) Using decaying learning rates or increasing the number of training samples per task $n$ can tighten the generalization bounds, with the other parameters held fixed. \\
(2) Smaller initial learning rates $c$ can reduce the generalization bounds, while their dependence on the local training length $N$ varies across learning-rate schedules. \\
(3) Positive cumulative alignment between the gradients of previous and subsequent tasks contributes to reducing the EF upper bound through the term $-\mathcal{G}_{T,i}$.

\section{Experiment}
\begin{figure*}[t]
    \centering
    \begin{subfigure}[b]{1\textwidth}
        \centering
        \includegraphics[width=\linewidth]{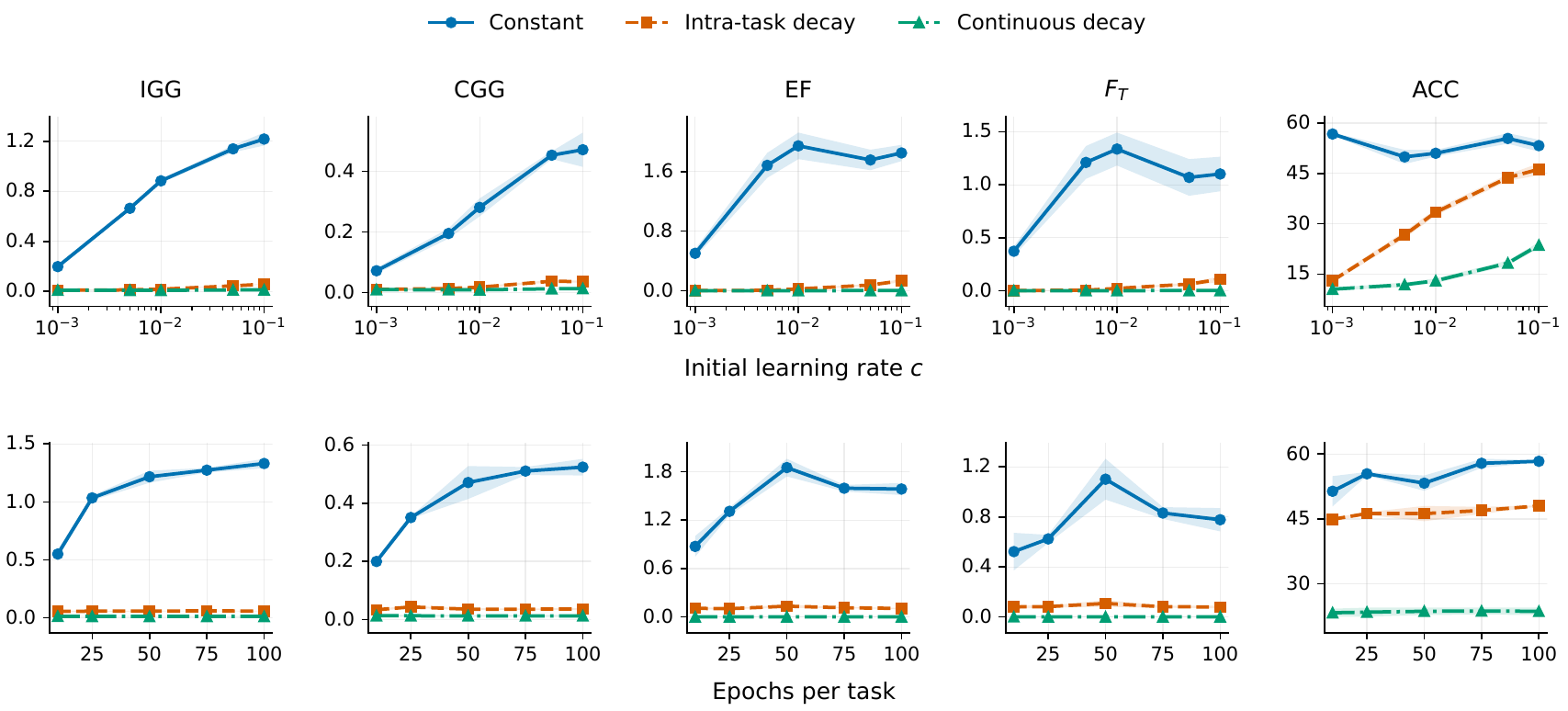}
    \end{subfigure}
    
    \caption{\textbf{IGG, CGG, EF, $F_T$, and ACC under different learning rate schedules on Split CIFAR-100.} The top row shows results for different initial learning rates $c$, while the bottom row shows results for different numbers of training epochs per task.}
\label{fig:lr_epochs}
\end{figure*}

In this section, we present experiments on Split CIFAR-100 to examine the empirical implications of our theoretical analysis. We also conduct experiments on Split MiniImageNet, with the results reported in Appendix~\ref{Appendix:experiments-on-miniimagenet}.

\subsection{Experimental Setup}

\textbf{Experimental setting on Split CIFAR-100. }
We conduct experiments on the widely used CL benchmark Split CIFAR-100~\citep{chaudhry2018efficient}. This benchmark partitions the 100 classes of CIFAR-100~\citep{krizhevsky2009learning} into 10 disjoint tasks, with 10 classes per task. We use AlexNet as the backbone, following the setting of \citet{serra2018overcoming}. We train the model with SGD to match our theoretical setting. To verify our theoretical findings, we vary the number of epochs per task in $\{10,25,50,75,100\}$, the initial learning rate in $\{10^{-3},5\times10^{-3},10^{-2},5\times10^{-2},10^{-1}\}$, and the number of samples per task in $\{500,1500,2500,3500,4500\}$. To verify how the number of epochs per task affects the metrics, we fix \(c=0.1\) and use the original training set setting, same as in Split CIFAR-100. To verify how the initial learning rate affects the metrics, we fix the number of epochs per task to 50 and use the same training set setting. To examine the effect of the number of training samples per task, we fix $c=0.1$ and use a common budget of $N=3550$ updates per task. Each experiment is repeated with random seeds 1, 2, and 3, and the figures report the mean and standard deviation across the three runs.

\textbf{Observable metric. } The definition in Equation \eqref{Eqn: forgetting decomposed} relies on the population risk, which is not directly observable. Therefore, we define an observable metric by replacing the population risk in the theoretical metric with the test loss. Specifically, let $R_{t,i}$ and $\tilde{R}_{t,i}$ denote the training loss and test loss of the model on task $i$, respectively, after training on task $t$. We report the absolute generalization gaps and signed empirical and test forgetting as follows. These observed absolute gaps are distinct from the absolute expected gaps controlled by the stability theorems:
\[
\begin{aligned}
\text{IGG} &= \frac{1}{T-1}\sum_{i=1}^{T-1} |R_{i,i}-\tilde{R}_{i,i}|,\\
\text{CGG} &= \frac{1}{T-1}\sum_{i=1}^{T-1} |\tilde{R}_{T,i}-R_{T,i}|,\\
\text{EF}  &= \frac{1}{T-1}\sum_{i=1}^{T-1} (R_{T,i}-R_{i,i}),\\
\text{$F_T$} &= \frac{1}{T-1}\sum_{i=1}^{T-1} (\tilde{R}_{T,i}-\tilde{R}_{i,i}).
\end{aligned}
\]
In addition to forgetting, we evaluate the model's predictive performance using average accuracy (ACC). Let $A_{t,i}$ denote the test accuracy on task $i$ after learning the first $t$ tasks, and let $T$ be the total number of tasks. ACC is defined as
\[
\mathrm{ACC} = \frac{1}{T}\sum_{i=1}^{T} A_{T,i}.
\]

\textbf{Compute resources.  } We conduct the experiments on an Intel Xeon Platinum 8160 @ 2.10\,GHz system with an NVIDIA RTX A6000 GPU.

\subsection{Experimental Result}

\textbf{Effects of the learning rate and the number of epochs per task.} As shown in Fig.~\ref{fig:lr_epochs}, the two decaying learning-rate schedules generally yield smaller generalization gaps and less forgetting than the constant schedule. In particular, continuous decay keeps forgetting close to zero across the these settings. Moreover, the smallest initial learning rate and the fewest training epochs generally yield the least forgetting. To examine whether these observations extend to a more challenging benchmark and a different backbone, we also conduct experiments on Split MiniImageNet using ResNet-18. The results, reported in Appendix~\ref{Appendix:experiments-on-miniimagenet}, show similar overall trends.

\textbf{Stability-plasticity dilemma.}
As shown in Fig.~\ref{fig:lr_epochs}, training with a continuously decaying learning rate yields the lowest ACC among the three learning-rate schedules. Moreover, reducing the number of training epochs or the initial learning rate does not necessarily improve ACC, even when it reduces forgetting. These observations reflect the classical stability-plasticity dilemma in CL~\citep{grossberg2013adaptive, NEURIPS2025_37d9f191, pmlr-v306-zheng26ac}, highlighting the difficulty of preserving past knowledge while learning new tasks.

\begin{figure}[htbp]
    \centering
    \begin{subfigure}[b]{0.485\textwidth}
        \centering
        \includegraphics[width=\linewidth]{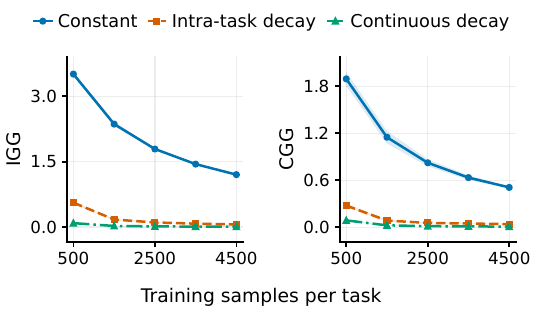}
    \end{subfigure}
   \caption{\textbf{IGG and CGG with different numbers of training samples per task on Split CIFAR-100.}}
\label{fig:samples}
\end{figure}

%\subsection{Generalization Gap}\label{Sec:GG}
\textbf{Effect of the number of training samples per task.} As shown in Fig.~\ref{fig:samples}, both IGG and CGG decrease as the number of training samples per task increases, with the number of SGD updates per task held fixed. This trend is consistent with our theoretical bounds, which decrease with the sample size when the other parameters are fixed.

\begin{figure}[htbp]
    \centering
    \begin{subfigure}[b]{0.485\textwidth}
        \centering
        \includegraphics[width=\linewidth]{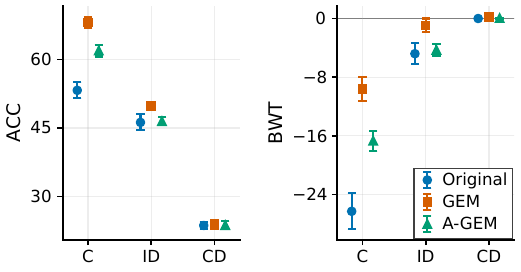}
    \end{subfigure}
    \caption{\textbf{ACC and BWT for original CL, GEM, and A-GEM under different learning-rate schedules on Split CIFAR-100.} Higher ACC indicates better predictive performance, while higher BWT indicates less forgetting.}
\label{fig:methods}
\end{figure}

\textbf{Effect of gradient alignment under different learning rate schedules.}
We compare original CL, GEM, and A-GEM on Split CIFAR-100 under three learning rate schedules. GEM enforces nonnegative inner products between the projected gradient and the memory gradients of  previous tasks, whereas A-GEM uses a single constraint based on an aggregate memory gradient. As shown in Fig.~\ref{fig:methods}, GEM achieves the highest ACC and BWT under the constant and intra-task decaying schedules, with larger performance gaps under the constant schedule. Under the continuously decaying schedule, however, all three methods exhibit similarly low ACC despite achieving BWT values close to zero. This pattern is consistent with the stability-plasticity dilemma. Overall, these results suggest that the benefits of gradient alignment depend on the learning rate schedule, in line with our theoretical analysis in Section~\ref{Sec: EF}, which relates EF to both CGA and learning rates.

\section{Conclusion}
In this work, we propose an algorithmic stability based theoretical framework for analyzing catastrophic forgetting under both convex and non-convex loss functions with three learning-rate schedules. Specifically, we derive a fine-grained decomposition of forgetting into IGG, CGG, and EF. We further propose intra-task and cross-task stability to analyze IGG and CGG, and analyze EF through the optimization dynamics. Our theoretical results characterize how local training lengths, initial learning rates, learning-rate schedules, and cumulative gradient alignment affect the forgetting bounds. Finally, we validate our theoretical findings through several empirical studies.

\textbf{Limitation.\hspace{1em}} The theoretical analysis is carried out under classical assumptions such as $L$-Lipschitz continuity and $\beta$-smoothness, which are standard in stability and generalization analyses. However, these assumptions may be restrictive in practice. It would be interesting to relax them and further strengthen our results in future work.

%\section*{Acknowledgment}

%\newpage
\bibliographystyle{IEEEtranN}
\bibliography{reference}

\begin{IEEEbiography}[{\includegraphics[width=1in,height=1.25in,clip,keepaspectratio]{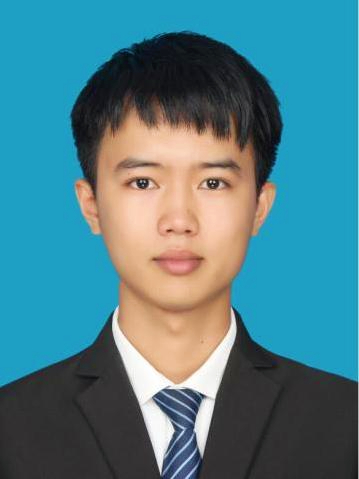}}]{Guodong Zheng} is currently a postdoctoral researcher with Sun Yat-sen University, China. He  received his Ph.D. degree from Huazhong University of Science and Technology, China.  His work appeared in several top-tier ML conferences/journals, including ICML, ICLR, and Neural Networks. His current research interests include optimization, continual learning. 
\end{IEEEbiography}

\vspace{-5mm}

\begin{IEEEbiography}[{\includegraphics[width=1in,height=1.25in,clip,keepaspectratio]{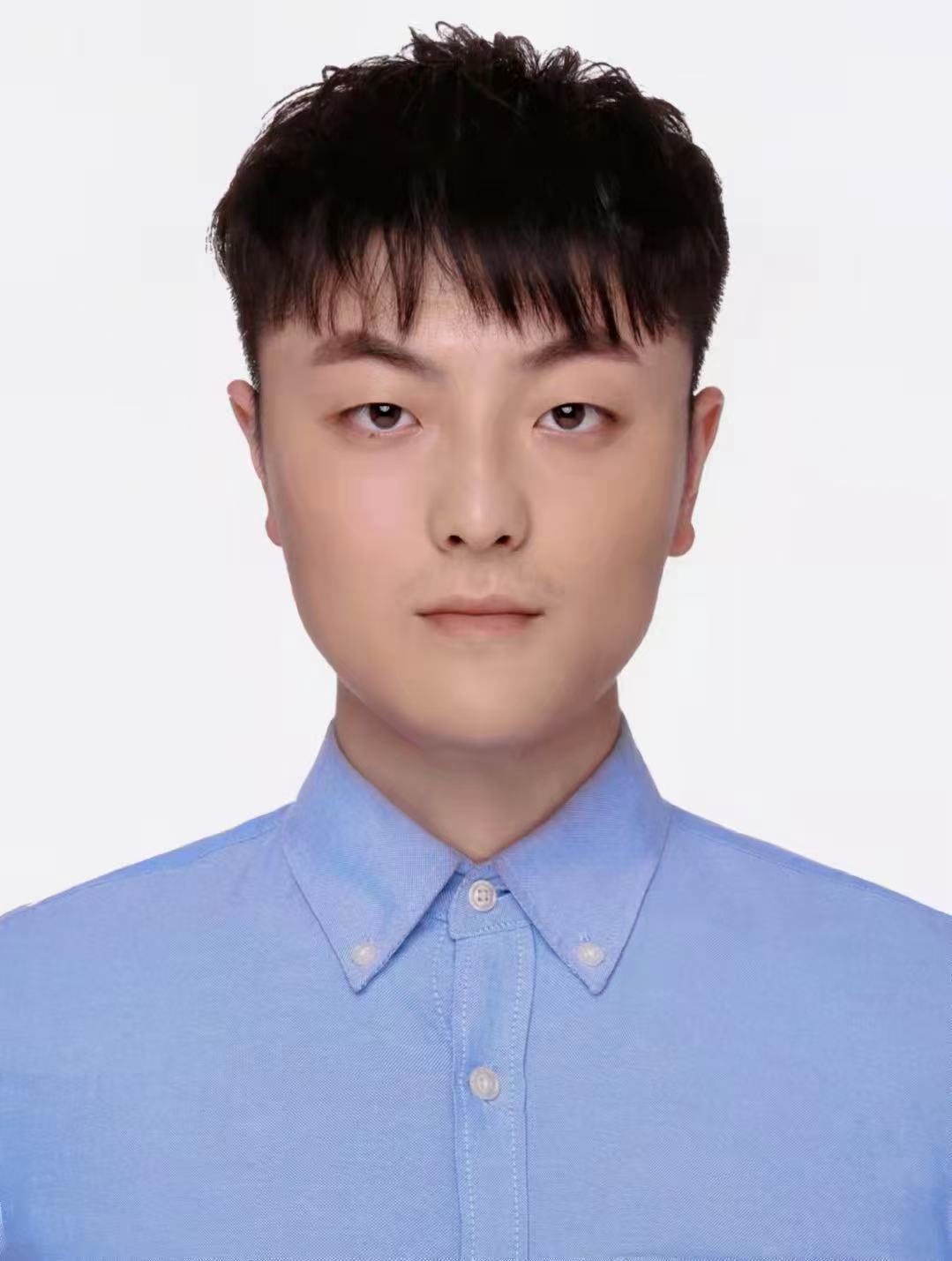}}]{Yan~Sun} is currently a Postdoctoral Researcher with Nanyang Technological University (NTU). He 
received the B.E. degree from the University of Science and Technology of China (USTC) and the Ph.D. degree from the University of Sydney. He has published 10+ first-author papers in leading machine learning venues, including NeurIPS, ICML, ICLR, and IEEE TPAMI. His research interests include machine learning, federated and distributed learning, optimization and generalization, and efficient training and adaptation of large-scale foundation models.
\end{IEEEbiography}

\vspace{-5mm}

\begin{IEEEbiography}[{\includegraphics[width=1in,height=1.25in,clip,keepaspectratio]{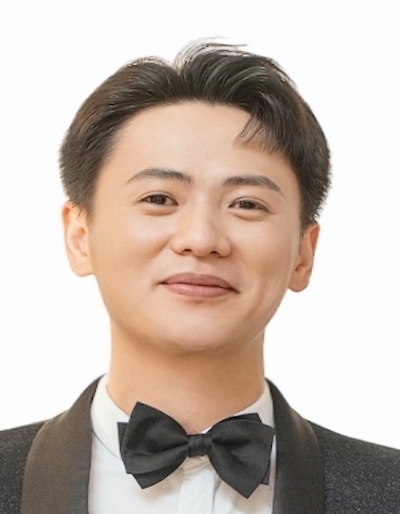}}]{Peng Wang} is currently a postdoctoral researcher with Sun Yat-sen University, China. He  received his Ph.D. degree from Huazhong University of Science and Technology, China. His work appeared in several top-tier ML conferences/journals, including ICML, ICLR, Neural Networks, and Pattern Recognition. His current research interests include the generalization and optimization theories of deep learning. 
\end{IEEEbiography}

\vspace{-5mm}

\begin{IEEEbiography}
[{\includegraphics[width=1in,height=1.25in,clip,keepaspectratio]{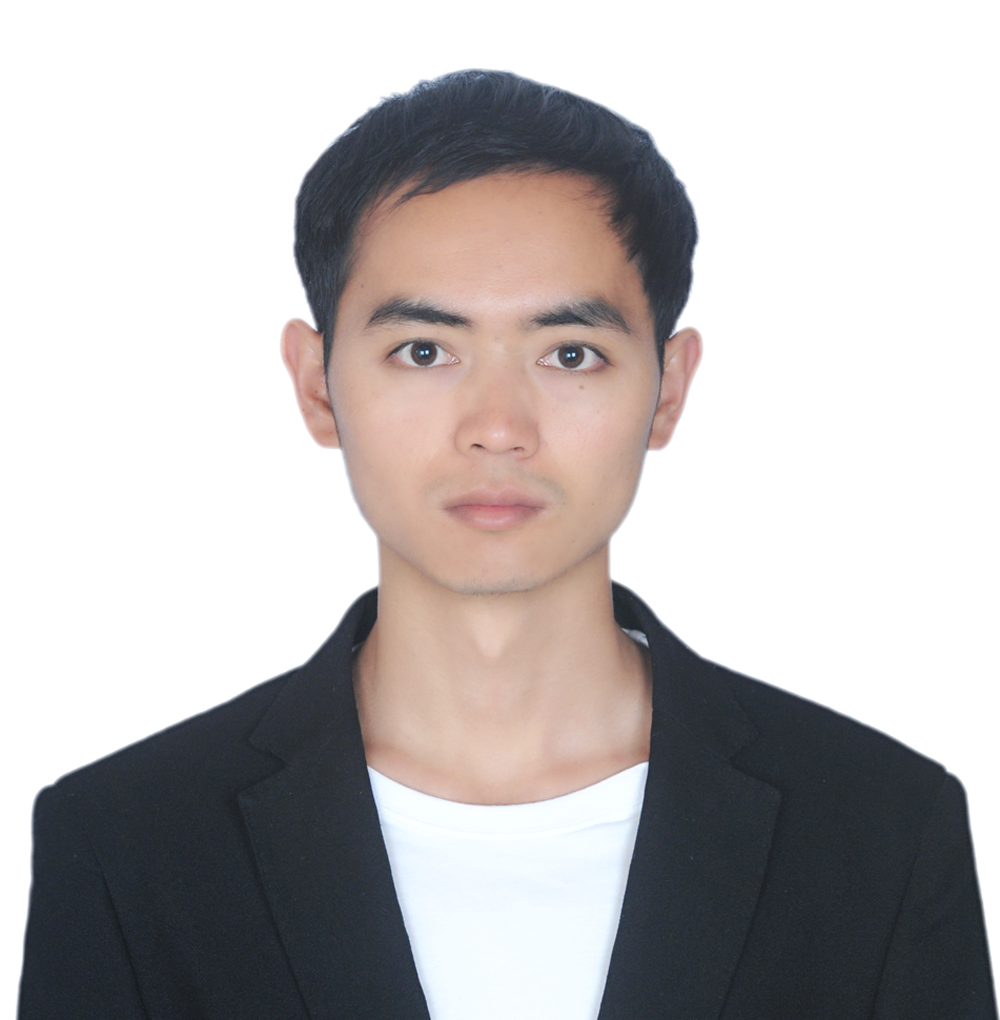}}]{Enneng~Yang} is currently a postdoctoral researcher at Sun Yat-sen University, China. He received his Ph.D. degree from the Software College, Northeastern University, China.
His work appeared in several top-tier ML \& IR conferences/journals, including IEEE TPAMI, IEEE TKDE, JMLR, ICML, NeurIPS, ICLR, ICCV, ACL, WWW, SIGIR, and AAAI. 
His research interests include machine learning and recommendation systems.
\end{IEEEbiography}

\vspace{-5mm}

\begin{IEEEbiography}[{\includegraphics[width=1in,height=1.25in,clip,keepaspectratio]{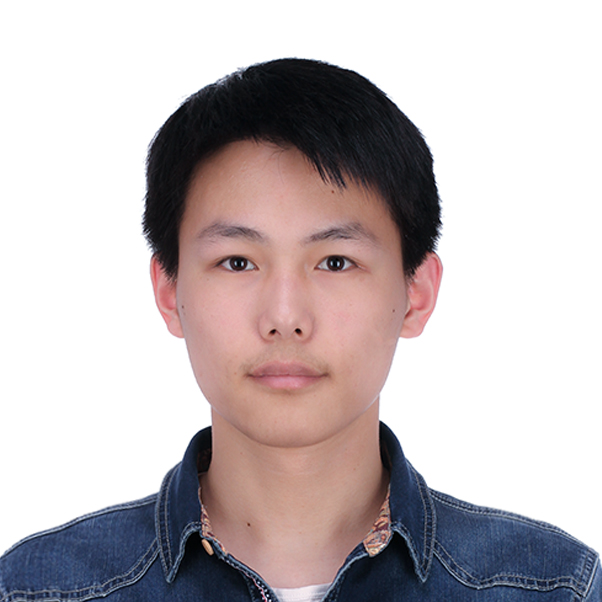}}]{Jiafei~Lyu} is currently a Senior Researcher at Tencent. He received the bachelor's and Ph.D. degrees from Tsinghua University in 2020 and 2025, respectively. He has published papers in NeurIPS, ICML, ICLR, AAAI, and the Journal of Artificial Intelligence Research.  His research interests include RLHF and efficient decision-making with deep Reinforcement Learning. 
\end{IEEEbiography}

\vspace{-5mm}

\begin{IEEEbiography}[{\includegraphics[width=1in,height=1.25in,clip,keepaspectratio]{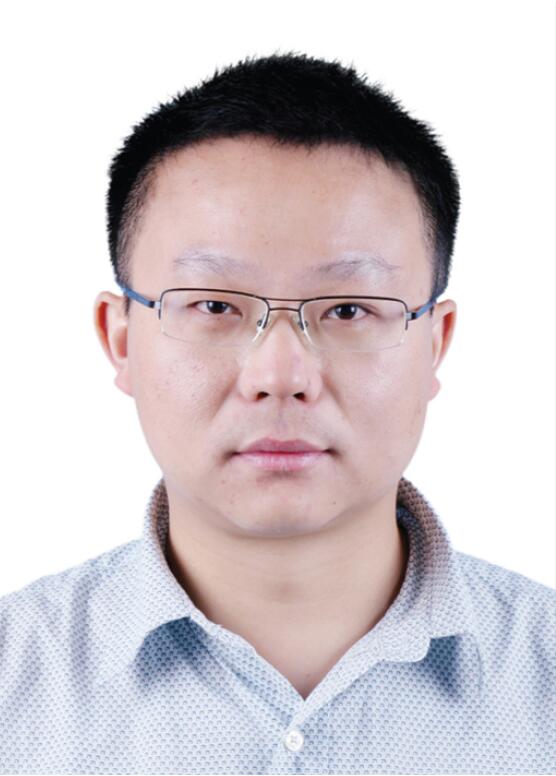}}]{Xiaochun Cao} (Senior Member, IEEE)  received the B.E. and M.E. degrees in Computer Science from Beihang University, Beijing, China, in 1999 and 2002, and the Ph.D. degree from the University of Central Florida, USA, in 2006. He worked as a Research Scientist at ObjectVideo Inc. for three years. From 2008 to 2012, he was a Professor at Tianjin University, China. Before joining Sun Yat-sen University, Shenzhen, he was a Professor at the Institute of Information Engineering, Chinese Academy of Sciences. He is currently the Dean and Professor with the School of Cyber Science and Technology at Sun Yat-sen University, Shenzhen. He has published over 200 papers and received the Piero Zamperoni Best Student Paper Award at \emph{ICPR} in 2004 and 2010. He serves on the Editorial Boards of \emph{IEEE TPAMI} and \emph{IEEE TIP}.
\end{IEEEbiography}

\vspace{-5mm}

\begin{IEEEbiography}[{\includegraphics[width=1in,height=1.25in,clip,keepaspectratio]{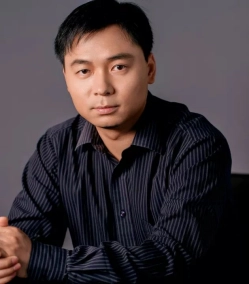}}]{Dacheng Tao} (Fellow, IEEE) is currently a Dis-
tinguished University Professor with the College of
Computing and Data Science, Nanyang Technologi-
cal University. He mainly applies statistics and math-
ematics to artificial intelligence and data science. His
research is detailed in one monograph and over 200
publications in prestigious journals and conferences,
with best paper awards, best student paper awards,
and test-of-time awards. His publications have been
cited over 112K times. He has an H-index of more
than 160 in Google Scholar. He is a Fellow of Aus-
tralian Academy of Science, AAAS, and ACM. He received the 2015 and 2020
Australian Eureka Prize, the 2018 IEEE ICDM Research Contributions Award,
and the 2021 IEEE Computer Society McCluskey Technical Achievement
Award.
\end{IEEEbiography}

\vspace{-5mm}

\begin{IEEEbiography}[{\includegraphics[width=1in,height=1.25in,clip,keepaspectratio]{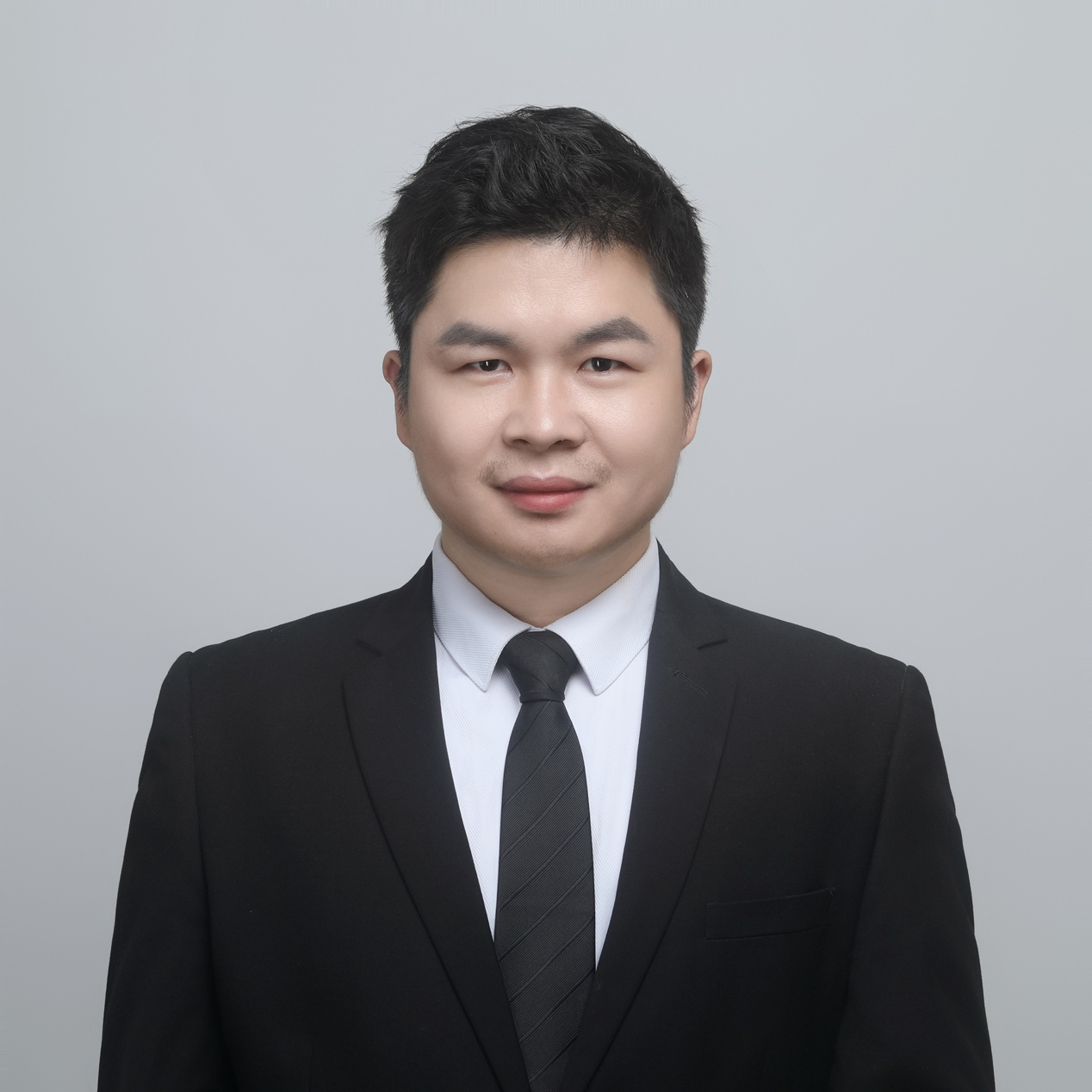}}]{Li~Shen} is currently an associate professor at Sun Yat-sen University. Previously, he was a research scientist at JD Explore Academy, Beijing, and a senior researcher at Tencent AI Lab, Shenzhen. He received his bachelor's degree and Ph.D. from the School of Mathematics, South China University of Technology. His research interests include efficient deep learning, efficient reinforcement learning, optimization and deep learning theory. He has served as the senior program committee for AAAI and IJCAI, area chairs for ICML, NeurIPS, ICLR, CVPR, and ACMMM. He is also the associate editor for IEEE Transactions on Pattern Analysis and Machine Intelligence, IEEE Transactions on Knowledge and Data Engineering, and IEEE Transactions on Multimedia. \end{IEEEbiography}

\newpage
\appendices
\section{Preliminaries on Algorithmic Stability}
In this section, we introduce several preliminaries that will be used in the subsequent proofs. We adopt the notation of \citet{hardt2016train}.

\begin{definition}
For a nonnegative step size $\alpha \geq 0$ and a function $f\colon \Omega \to \mathbb{R}$, we define the gradient update rule $G_{f,\alpha}$ as
\[
G_{f,\alpha}(w) = w - \alpha \nabla f(w).
\]
\end{definition}

\begin{definition}[$\eta-$expansive] An update rule is $\eta-$expansive if 
\begin{equation}
    \sup_{\substack{v,w\in\Omega\\v\ne w}}\frac{\|G(v)-G(w)\|}{\|v-w\|}\leq\eta
\end{equation}
\end{definition}

\begin{definition}[$\sigma-$bounded] An update rule is $\sigma-$bounded if 
\begin{equation}
    \sup_{w\in\Omega}{\|w-G(w)\|}\leq\sigma
\end{equation}
\end{definition}

\begin{lemma}\label{Lem: f property}
Assume that $f$ is $\beta$-smooth. Then, the following properties hold.\\
(1). $G_{f,\alpha}$ is $(1 + \alpha\beta)$-expansive.\\
(2). Assume in addition that $f$ is convex. Then, for any $\alpha \leq 2/\beta$, the gradient update $G_{f,\alpha}$ is $1$-expansive.\\
(3). Assume in addition that $f$ is $\gamma$-strongly convex. Then, for $\alpha \leq \frac{2}{\beta+\gamma}$, $G_{f,\alpha}$ is $\left(1 - \frac{\alpha\beta\gamma}{\beta+\gamma}\right)$-expansive.
\end{lemma}

\section{Intra-Task Generalization Gap}

\subsection{Proof of Theorem \ref{Thm: IGG}}\label{Sec: Proof of IGG}
\begin{theorem*}
If $\mathcal{A}$ is $\epsilon_{i,i}$-uniformly stable, then
    \begin{equation}
        \begin{aligned}
        &\big|\mathbb{E}_{S_{1:i},\mathcal{A}}[R_{\mathcal{D}_i}(\mathcal{A}(S_{1:i}))-R_{S_i}(\mathcal{A}(S_{1:i}))]\big|\leq\epsilon_{i,i}.
        \end{aligned}
    \end{equation}
\end{theorem*}

\begin{proof}
We first introduce the following notation. Let $S_{1:i}$ and $S_{1:i}'$ be two independent datasets, and let $S_{1:i}^{(i,k)}$ denote the sample obtained from $S_{1:i}$ by replacing the $k$-th example $z_{i,k}$ in the dataset $S_i$ from task $i$ with $z_{i,k}'$. More explicitly,
\begin{equation}
\begin{aligned}
    S_{1:i} &= (z_{1,1},\dots,z_{1,n},\dots,z_{i,1},\dots,z_{i,n}),\\
    S_{1:i}' &= (z_{1,1}',\dots,z_{1,n}',\dots,z_{i,1}',\dots,z_{i,n}'),\\
    S_{1:i}^{(i,k)} &= (z_{1,1},\dots,z_{1,n},\dots,z_{i,1},\dots,z_{i,k-1},\\
    &\qquad z_{i,k}',z_{i,k+1},\dots,z_{i,n}).
\end{aligned}
\end{equation}

With these notations, we have
\begin{equation*}
\begin{aligned}
    &\mathbb{E}_{S_{1:i}}\mathbb{E}_\mathcal{A}[R_{S_i}(\mathcal{A}(S_{1:i}))]\\
    &= \mathbb{E}_{S_{1:i}}\mathbb{E}_\mathcal{A}\Bigg[\frac{1}{n}\sum_{k=1}^n f(\mathcal{A}(S_{1:i});z_{i,k})\Bigg] \\
    &= \mathbb{E}_{S_{1:i}}\mathbb{E}_{S_{1:i}'}\mathbb{E}_\mathcal{A}\Bigg[\frac{1}{n}\sum_{k=1}^n f(\mathcal{A}(S_{1:i}^{(i,k)});z_{i,k}')\Bigg] \\
    &= \mathbb{E}_{S_{1:i}}\mathbb{E}_{S_{1:i}'}\mathbb{E}_\mathcal{A}\Bigg[\frac{1}{n}\sum_{k=1}^n f(\mathcal{A}(S_{1:i});z_{i,k}')\Bigg] + \delta \\
    &= \mathbb{E}_{S_{1:i}}\mathbb{E}_\mathcal{A}[R_{\mathcal{D}_i}(\mathcal{A}(S_{1:i}))] + \delta.
\end{aligned}
\end{equation*}
where
\[
\begin{aligned}
\delta={}&\mathbb{E}_{S_{1:i}}\mathbb{E}_{S_{1:i}'}\mathbb{E}_\mathcal{A}\Bigg[\frac{1}{n}\sum_{k=1}^n f(\mathcal{A}(S_{1:i}^{(i,k)});z_{i,k}')\\
&-\frac{1}{n}\sum_{k=1}^n f(\mathcal{A}(S_{1:i});z_{i,k}')\Bigg].
\end{aligned}
\]

Under the assumption that $\mathcal{A}$ is $\epsilon_{i,i}$-uniformly stable, we have
\[
|\delta|\leq \epsilon_{i,i}.
\]
Therefore,
\[
\begin{aligned}
&\big|\mathbb{E}_{S_{1:i},\mathcal{A}}[R_{\mathcal{D}_i}(\mathcal{A}(S_{1:i}))-R_{S_i}(\mathcal{A}(S_{1:i}))]\big|
= |\delta|\leq \epsilon_{i,i}.
\end{aligned}
\]
\end{proof}

\subsection{Proof of Theorem \ref{Thm: IGG bound convex}}\label{Sec: proof of IGG bound convex}

\begin{theorem*}
Assume that the loss function $f(\cdot;z)$ satisfies Assumptions~\ref{assump: Lipschitz}, \ref{assump: smooth}, and \ref{assump: convex}. Suppose that SGD is run for $N$ steps on each task $\tau\in[i]$ with step sizes $\eta_{\tau,t}\leq \frac{2}{\beta}$. Then,
\[
\epsilon_{i,i}\leq \frac{2L^2}{n}\sum_{t=0}^{N-1} \eta_{i,t}.
\]
\end{theorem*}

\begin{proof}
Let $S_{1:i}$ and $S_{1:i}'$ be two datasets consisting of samples from the first $i$ tasks, and assume that they differ in only one example in task $i$. Consider the update sequences $G_{1,0},\dots,G_{1,N-1},\dots,G_{i,0},\dots,G_{i,N-1}$ and $G'_{1,0},\dots,G'_{1,N-1},\dots,G'_{i,0},\dots,G'_{i,N-1}$ induced by running SGD on $S_{1:i}$ and $S_{1:i}'$, respectively. The two runs use the same initialization and sampled indices.

Fix an arbitrary example $z\in\mathcal{Z}$. By the $L$-Lipschitz property of $f(\cdot;z)$, we have
\begin{equation}\label{eq:lipschitz_bound_iN}
    \mathbb{E}\big|f(w_{i,N};z)-f(w'_{i,N};z)\big|
    \leq L\,\mathbb{E}[\delta_{i,N}],
\end{equation}
where $\delta_{i,N}=\|w_{i,N}-w'_{i,N}\|$.

During the $t$-th step of training on task $i$, with probability $1-\frac{1}{n}$, SGD selects the same example in both $S_i$ and $S_i'$. In this case, the gradient update operators coincide, namely $G_{i,t}=G'_{i,t}$, and we use the $1$-expansivity of $G_{i,t}$, which follows from the convexity and smoothness of $f$. With probability $\frac{1}{n}$, the selected examples differ. In this case, both $G_{i,t}$ and $G'_{i,t}$ are $\eta_{i,t}L$-bounded.

Combining these two cases yields
\begin{equation*}
    \begin{aligned}
    \mathbb{E}[\delta_{i,t+1}]
    &\leq \left(1-\frac{1}{n}\right)\mathbb{E}[\delta_{i,t}]
    + \frac{1}{n}\mathbb{E}[\delta_{i,t}]
    + \frac{2\eta_{i,t}L}{n}\\
    &= \mathbb{E}[\delta_{i,t}] + \frac{2\eta_{i,t}L}{n}.
    \end{aligned}
\end{equation*}
Unrolling this recursion over the $N$ steps of task $i$, we obtain
\begin{equation}\label{eq:task_i_recursion}
    \mathbb{E}[\delta_{i,N}]
    \leq \mathbb{E}[\delta_{i,0}]
    + \sum_{t=0}^{N-1} \frac{2\eta_{i,t}L}{n}.
\end{equation}

Under the CL protocol, the initialization of task $i$ is the final model from task $i-1$, so $w_{i,0}=w_{i-1,N}$ and $w'_{i,0}=w'_{i-1,N}$. Therefore,
\[
\delta_{i,0}=\|w_{i,0}-w'_{i,0}\|=\|w_{i-1,N}-w'_{i-1,N}\|=\delta_{i-1,N}.
\]
Since the perturbation occurs only in task $i$, all previous tasks are identical in $S_{1:i-1}$ and $S_{1:i-1}'$. Hence, $w_{i-1,N}=w'_{i-1,N}$, which implies $\delta_{i-1,N}=0$. Therefore,
\[
\mathbb{E}[\delta_{i,0}]=0.
\]
Substituting this into Equation  \eqref{eq:task_i_recursion} gives
\[
\mathbb{E}[\delta_{i,N}]
\leq \frac{2L}{n}\sum_{t=0}^{N-1} \eta_{i,t}.
\]

Finally, plugging this bound into Equation \eqref{eq:lipschitz_bound_iN}, we obtain
\[
\mathbb{E}\big|f(w_{i,N};z)-f(w'_{i,N};z)\big|
\leq \frac{2L^2}{n}\sum_{t=0}^{N-1} \eta_{i,t}.
\]
This proves the theorem.
\end{proof}

\subsection{Proof of Theorem \ref{Thm: IGG bound non-convex}}\label{Sec: proof of IGG bound non-convex}

\begin{lemma}\label{Lem: non-convex-upper-bound}
Assume that $n\geq2$ and that the loss function $f(\cdot;z)$ is nonnegative and $L$-Lipschitz for every $z$. Let $S_{1:i}$ and $S_{1:i}'$ be two datasets consisting of samples from the first $i$ tasks, differing in only one example in task $i$. Then, for every $z\in\mathcal{Z}$ and every $t_0\in\{0,1,\ldots,N\}$, under the random update rule, we have
\[
\begin{aligned}
&\mathbb{E}\big|f(w_i;z)-f(w_i';z)\big|\\
 \leq&\frac{t_0}{n}\sup_{w,z}f(w;z)
+
L\,\mathbb{E}\!\left[\|w_i-w_i'\|\mid \delta_{i,t_0}=0\right].
\end{aligned}
\]
The same inequality holds under the random permutation rule when $t_0<n$.
\end{lemma}

\begin{proof}
Let $S_{1:i}$ and $S_{1:i}'$ be two datasets consisting of samples from the first $i$ tasks, and assume that they differ in only one example in task $i$. Consider the update sequences $G_{1,0},\dots,G_{1,N-1},\dots,G_{i,0},\dots,G_{i,N-1}$ and $G'_{1,0},\dots,G'_{1,N-1},\dots,G'_{i,0},\dots,G'_{i,N-1}$ induced by running SGD on $S_{1:i}$ and $S_{1:i}'$, respectively, with the same initialization and sampled indices.

Let $\mathcal{E}$ denote the event that $\delta_{i,t_0}=0$. Then,
\begin{equation*}
\begin{aligned}
    &\mathbb{E}\big|f(w_i;z)-f(w_i';z)\big|\\
    =& \mathbb{P}(\mathcal{E})\,\mathbb{E}\big[|f(w_i;z)-f(w_i';z)|\mid \mathcal{E}\big] \\
    & + \mathbb{P}(\mathcal{E}^c)\,\mathbb{E}\big[|f(w_i;z)-f(w_i';z)|\mid \mathcal{E}^c\big] \\
    \leq& \mathbb{E}\big[|f(w_i;z)-f(w_i';z)|\mid \mathcal{E}\big]+ \mathbb{P}(\mathcal{E}^c)\sup_{w,z} f(w;z) \\
    \leq& L\,\mathbb{E}\big[\|w_i-w_i'\|\mid \mathcal{E}\big]
    + \mathbb{P}(\mathcal{E}^c)\sup_{w,z} f(w;z).
\end{aligned}
\end{equation*}

We now bound $\mathbb{P}(\mathcal{E}^c)$. Let $i^*$ denote the unique position at which $S_{1:i}$ and $S_{1:i}'$ differ, and let $I$ be the first step at which SGD selects the example $z_{i^*}$. If $I>t_0$, then the two executions are identical up to step $t_0$, and thus $\delta_{i,t_0}=0$. Therefore,
\[
\mathbb{P}(\mathcal{E}^c)=\mathbb{P}(\delta_{i,t_0}\neq 0)\leq \mathbb{P}(I\leq t_0).
\]

Under the random permutation rule, $I$ is uniformly distributed over $\{1,\ldots,n\}$, so, for $t_0<n$,
\[
\mathbb{P}(I\leq t_0)=\frac{t_0}{n}.
\]

Under the random selection rule, let $A_t$ be the event that $z_{i^*}$ is selected at step $t$. Then
\[
\mathbb{P}(I\leq t_0)\leq \mathbb{P}\Big(\bigcup_{t=1}^{t_0} A_t\Big)
\leq \sum_{t=1}^{t_0}\mathbb{P}(A_t)=\frac{t_0}{n}.
\]
This proves the claim.
\end{proof}

\begin{theorem*}
Assume that the loss function $f(\cdot;z)\in[0,1]$ satisfies Assumptions~\ref{assump: Lipschitz} and \ref{assump: smooth}, with $n\geq2$ and $c,\beta>0$. Suppose that SGD is run for $N$ steps on each task $\tau\in[i]$ with learning rates $\eta_{\tau,t}$. Then,

(a) under intra-task decaying learning rates $\eta_{\tau,t}=\frac{c}{t+1}$:
    \[
    \epsilon_{i,i}\leq 
    \begin{cases} 
          \mathcal{K}_N-\frac{2L^2}{\beta n},&\text{if }2cL^2<N,\\[1ex]
          \frac{N}{n},&\text{if }2cL^2\geq N.
    \end{cases}
    \]
    
(b) under continuously decaying learning rates $\eta_{\tau,t}=\frac{c}{(\tau-1) N+t+1}$ and $i>1$:
    \[
    \epsilon_{i,i}\leq 
    \begin{cases} 
          \frac{2L^2}{\beta n}\left[\left(\frac{i}{i-1}\right)^{c\beta}-1\right],&\text{if }2cL^2\leq \mathcal{C}_1^{\mathrm{CD}},\\[1ex]
          \mathcal{K}_{iN}-\frac{(i-1)N}{n}-\frac{2L^2}{\beta n},
          &\text{if }\mathcal{C}_1^{\mathrm{CD}}<2cL^2<iN,\\[1ex]
          \frac{N}{n},&\text{if }2cL^2\geq iN,
    \end{cases}
    \]
    where 
    \[
    \begin{aligned}
        \mathcal{C}_1^{\mathrm{CD}}&=(i-1)N\left(\frac{i-1}{i}\right)^{c\beta}.
    \end{aligned}
    \]
\end{theorem*}

\begin{proof}
    Let $S_{1:i}$ and $S_{1:i}'$ be two datasets consisting of samples from the first $i$ tasks, and assume that they differ in only one example in task $i$. Consider the update sequences $G_{1,0},\dots,G_{1,N-1},\dots,G_{i,0},\dots,G_{i,N-1}$ and $G'_{1,0},\dots,G'_{1,N-1},\dots,G'_{i,0},\dots,G'_{i,N-1}$ induced by running SGD on $S_{1:i}$ and $S_{1:i}'$, respectively. The two runs use the same initialization and sampled indices.
    
    Based on Lemma \ref{Lem: non-convex-upper-bound}, we have for every integer $t_0\in\{0,\ldots,N\}$ under independent uniform sampling with replacement,
    \[
    \begin{aligned}
    &\mathbb{E}\abs{f(w_i;z)-f(w_i';z)}\leq\frac{t_0}{n}+L\mathbb{E}[\|w_i-w_i'\|\mid\delta_{i,t_0}=0].
    \end{aligned}
    \]

    We let $\Delta_{i,t}=\mathbb{E}[\delta_{i,t}|\delta_{i,t_0}=0].$ The case $L=0$ is immediate, so we take $L>0$ in the calculations below. Next, we consider two types of learning rates.

    Integer cutoffs are handled by the following comparison for any nonnegative, nonincreasing function $D$ on $(0,N]$:
    \begin{equation}\label{Eqn: integer-cutoff comparison}
    \begin{aligned}
    \min_{m\in\{0,\ldots,N\}}\left\{m+\sum_{t=m+1}^{N}D(t)\right\}&=\sum_{t=1}^{N}\min\{1,D(t)\}\\
    &\leq\int_0^N\min\{1,D(u)\}\,du.
    \end{aligned}
    \end{equation}
    Choose $m$ as the number of indices for which $D(t)\geq1$ to obtain the equality; monotonicity then gives the integral bound. We use this comparison to evaluate the cutoff through a finite integral.

    (a) Under intra-task decaying learning rates $\eta_{\tau,t}=\frac{c}{t+1}$, the $t$-th update uses $\eta_{i,t-1}=c/t$. For $t=t_0+1,\ldots,N$,
    \begin{equation*}
    \begin{aligned}
        \Delta_{i,t}
        \leq&(1-\frac{1}{n})(1+\eta_{i,t-1}\beta)\Delta_{i,t-1}+\frac{1}{n}\Delta_{i,t-1}+\frac{2\eta_{i,t-1}L}{n}\\
        =&(\frac{1}{n}+(1-\frac{1}{n})(1+\frac{c\beta}{t}))\Delta_{i,t-1}+\frac{2cL}{tn}\\
        =&(1+(1-\frac{1}{n})\frac{c\beta}{t})\Delta_{i,t-1}+\frac{2cL}{tn}\\
        \leq&\exp((1-\frac{1}{n})\frac{c\beta}{t})\Delta_{i,t-1}+\frac{2cL}{tn}.
    \end{aligned} 
    \end{equation*}

    Unrolling this recursion from $t_0$ to $N$ yields the following bound, where the integral estimate is used for $t_0\geq1$:
    \begin{equation}\label{Eqn: Delta_iN_ID_LR}
    \begin{aligned}
        \Delta_{i,N}
        \leq&\sum_{t=t_0+1}^N\Bigg\{\prod_{k=t+1}^N\exp((1-\frac{1}{n})\frac{c\beta }{k})\Bigg\}\frac{2cL}{tn}\\
        =&\sum_{t=t_0+1}^N\Bigg\{\exp\Big((1-\frac{1}{n})c\beta \sum_{k=t+1}^N\frac{1}{k}\Big)\Bigg\}\frac{2cL}{tn}\\
        \leq&\sum_{t=t_0+1}^N\Bigg\{\exp\Big((1-\frac{1}{n})c\beta \ln(\frac{N}{t})\Big)\Bigg\}\frac{2cL}{tn}\\
        =&\frac{2cL}{n}N^{(1-\frac{1}{n})c\beta }\sum_{t=t_0+1}^N t^{-(1-\frac{1}{n})c\beta -1}\\
        \leq&\frac{2cL}{n}N^{c\beta}\sum_{t=t_0+1}^{N}t^{-c\beta-1}\\
        \leq&\frac{2cL}{n}N^{c\beta}\int_{t_0}^{N}u^{-c\beta-1}\,du\\
        =&\frac{2L}{\beta n}\left[\left(\frac{N}{t_0}\right)^{c\beta}-1\right].
    \end{aligned} 
    \end{equation}
    Here, the first inequality after the sum uses $(N/t)^{(1-1/n)c\beta}\leq(N/t)^{c\beta}$. Therefore, we obtain
    \[
    \begin{aligned}
    &\mathbb{E}\left|f(w_i;z)-f(w_i';z)\right|\leq\frac{t_0}{n}+\frac{2L^2}{\beta n}
    \left[\left(\frac{N}{t_0}\right)^{c\beta}-1\right].
    \end{aligned}
    \]
    
    Differentiating the right-hand side with respect to the real variable $t_0$ gives the stationary point
    \[
        t_0^\star=(2cL^2)^\frac{1}{c\beta+1}N^\frac{c\beta}{c\beta+1}.
    \]
    
    Applying Equation \eqref{Eqn: integer-cutoff comparison} to the unrolled sum, with $D(u)=2cL^2N^{c\beta}u^{-c\beta-1}$, gives
    \[
    \begin{aligned}
    &\mathbb{E}|f(w_i;z)-f(w_i';z)|\\
    \leq&\frac1n\int_0^N\min\{1,2cL^2N^{c\beta}u^{-c\beta-1}\}\,du.
    \end{aligned}
    \]
    When $t_0^\star<N$, splitting the integral at $t_0^\star$ yields
    \[
    \begin{aligned}
    \mathbb{E}\left|f(w_i;z)-f(w_i';z)\right|
    \leq&\frac{t_0^\star}{n}
    +\frac{2cL^2N^{c\beta}}{n}\int_{t_0^\star}^{N}u^{-c\beta-1}\,du\\
    =&\frac{c\beta+1}{nc\beta}t_0^\star-\frac{2L^2}{\beta n}.
    \end{aligned}
    \]
    When $t_0^\star\geq N$, the integrand equals $1$ throughout $(0,N)$, giving $N/n$. Since $t_0^\star<N$ is equivalent to $2cL^2<N$, substituting $t_0^\star$ gives part (a).
    
    (b) Under continuously decaying learning rates $\eta_{\tau,t}=\frac{c}{(\tau-1)N+t+1}$, the $t$-th update uses $\eta_{i,t-1}=c/((i-1)N+t)$. For $t=t_0+1,\ldots,N$,
    \begin{equation*}
    \begin{aligned}
        &\Delta_{i,t}\\
        \leq&\Big(1-\frac{1}{n}\Big)(1+\eta_{i,t-1}\beta)\Delta_{i,t-1}+\frac{1}{n}\Delta_{i,t-1}+\frac{2\eta_{i,t-1}L}{n}\\
        =&\bigg(\frac{1}{n}+\Big(1-\frac{1}{n}\Big)\Big(1+\frac{c\beta}{(i-1)N+t}\Big)\bigg)\Delta_{i,t-1}\\
        &+\frac{2cL}{((i-1)N+t)n}\\
        =&\bigg(1+\Big(1-\frac{1}{n}\Big)\frac{c\beta}{(i-1)N+t}\bigg)\Delta_{i,t-1}+\frac{2cL}{((i-1)N+t)n}\\
        \leq&\exp{\bigg(\Big(1-\frac{1}{n}\Big)\frac{c\beta}{(i-1)N+t}\bigg)}\Delta_{i,t-1}+\frac{2cL}{((i-1)N+t)n}.
    \end{aligned} 
    \end{equation*}

    Unrolling this recursion from $t_0$ to $N$ yields
    \begin{equation}\label{Eqn: Delta_iN_CD_LR}
    \begin{aligned}
    &\Delta_{i,N}\\
    \leq& \sum_{t=t_0+1}^N \Bigg\{\prod_{k=t+1}^N \exp\bigg( \Big(1-\frac{1}{n}\Big)\frac{c\beta }{(i-1)N+k} \bigg) \Bigg\}\\
    &\qquad\times\frac{2cL}{((i-1)N+t)n} \\
    =& \sum_{t=t_0+1}^N \exp\bigg( \Big(1-\frac{1}{n}\Big)c\beta  \sum_{k=t+1}^N \frac{1}{(i-1)N+k} \bigg)\\
    &\qquad\times\frac{2cL}{((i-1)N+t)n} \\
    \leq& \sum_{t=t_0+1}^N \exp\bigg( \Big(1-\frac{1}{n}\Big)c\beta  \ln\bigg( \frac{iN}{(i-1)N+t} \bigg) \bigg)\\
    &\qquad\times\frac{2cL}{((i-1)N+t)n} \\
    =& \frac{2cL}{n} (iN)^{(1-\frac{1}{n})c\beta }\sum_{t=t_0+1}^N \big((i-1)N+t\big)^{-(1-\frac{1}{n})c\beta  - 1} \\
    \leq&\frac{2cL}{n}(iN)^{c\beta}
    \sum_{t=t_0+1}^{N}\big((i-1)N+t\big)^{-c\beta-1}\\
    \leq&\frac{2cL}{n}(iN)^{c\beta}
    \int_{t_0}^{N}\big((i-1)N+u\big)^{-c\beta-1}\,du\\
    =&\frac{2L}{\beta n}
    \left[\left(\frac{iN}{(i-1)N+t_0}\right)^{c\beta}-1\right],
    \end{aligned} 
    \end{equation}
    where we again use $iN/((i-1)N+t)\geq1$ to replace the exponent $(1-1/n)c\beta$ by $c\beta$ before taking the integral. Since $i>1$, this integral is finite also at $t_0=0$.

    Therefore, we have
    \begin{equation}\label{Eqn:bound-IGG-CD}
    \begin{aligned}
    &\mathbb{E}\left|f(w_i;z)-f(w_i';z)\right|\\
    \leq&\frac{t_0}{n}+\frac{2L^2}{\beta n}
    \left[\left(\frac{iN}{(i-1)N+t_0}\right)^{c\beta}-1\right].
    \end{aligned}
    \end{equation}

    Differentiating the right-hand side with respect to the real variable $t_0$ gives the stationary point
    \[
    \begin{aligned}
    t_0^\star={}&(2cL^2)^{\frac{1}{c\beta+1}}(iN)^{\frac{c\beta}{c\beta+1}}-(i-1)N.
    \end{aligned}
    \]
    Applying Equation  \eqref{Eqn: integer-cutoff comparison} to the unrolled sum, with $D(u)=2cL^2(iN)^{c\beta}((i-1)N+u)^{-c\beta-1}$, gives
    \[
    \begin{aligned}
    &\mathbb{E}|f(w_i;z)-f(w_i';z)|\\
    \leq&\frac1n\int_0^N
    \min\left\{1,\frac{2cL^2(iN)^{c\beta}}{((i-1)N+u)^{c\beta+1}}\right\}\,du.
    \end{aligned}
    \]
    When $t_0^\star\leq0$, the second argument of the minimum is at most $1$ throughout $(0,N)$, and evaluating the integral gives
    \[
    \begin{aligned}
    &\mathbb{E}|f(w_i;z)-f(w_i';z)|\leq\frac{2L^2}{\beta n}
    \left[\left(\frac{i}{i-1}\right)^{c\beta}-1\right].
    \end{aligned}
    \]
    When $0<t_0^\star<N$, splitting the integral at $t_0^\star$ gives
    \[
    \begin{aligned}
    &\mathbb{E}|f(w_i;z)-f(w_i';z)|\\
    \leq&\frac{t_0^\star}{n}+\frac{2cL^2(iN)^{c\beta}}{n}\int_{t_0^\star}^{N}
    \big((i-1)N+u\big)^{-c\beta-1}\,du\\
    =&\frac{c\beta+1}{nc\beta}\big((i-1)N+t_0^\star\big)-\frac{(i-1)N}{n}-\frac{2L^2}{\beta n}.
    \end{aligned}
    \]
    When $t_0^\star\geq N$, the integrand equals $1$ throughout $(0,N)$, giving $N/n$. The conditions $t_0^\star\leq0$ and $t_0^\star\geq N$ are equivalent to $2cL^2\leq(i-1)N((i-1)/i)^{c\beta}$ and $2cL^2\geq iN$, respectively. Substituting $t_0^\star$ gives part (b). The formulas agree at both boundaries.
\end{proof}

\section{Cross-Task Generalization Gap}

\subsection{Proof of Theorem \ref{Thm: CGG}}\label{Sec: Proof of CGG}

\begin{theorem*} 
If $\mathcal{A}$ is $\epsilon_{T,i}$-uniformly stable, then
    \begin{equation}
        |\mathbb{E}_{S_{1:T},\mathcal{A}}[R_{\mathcal{D}_i}(\mathcal{A}(S_{1:T}))-R_{S_i}(\mathcal{A}(S_{1:T}))]|\leq\epsilon_{T,i}.
    \end{equation}
\end{theorem*}

\begin{proof}
We first introduce the following notation. Let $S_{1:T}$ and $S_{1:T}'$ be two independent samples, and let $S_{1:T}^{(i,k)}$ denote the sample obtained from $S_{1:T}$ by replacing the $k$-th example $z_{i,k}$ in the dataset $S_i$ from task $i$ with $z_{i,k}'$. More explicitly,
\begin{equation}
\begin{aligned}
    S_{1:T} &= (z_{1,1},\dots,z_{1,n},\dots,z_{T,1},\dots,z_{T,n}),\\
    S_{1:T}' &= (z_{1,1}',\dots,z_{1,n}',\dots,z_{T,1}',\dots,z_{T,n}'),\\
    S_{1:T}^{(i,k)} &= (z_{1,1},\dots,z_{i,k-1},z_{i,k}',z_{i,k+1},\dots,z_{T,n}).
\end{aligned}
\end{equation}

With these notations, we have
\begin{equation*}
\begin{aligned}
    &\mathbb{E}_{S_{1:T}}\mathbb{E}_\mathcal{A}[R_{S_i}(\mathcal{A}(S_{1:T}))] \\
    =& 
    \mathbb{E}_{S_{1:T}}\mathbb{E}_\mathcal{A}\Bigg[\frac{1}{n}\sum_{k=1}^n f(\mathcal{A}(S_{1:T});z_{i,k})\Bigg] \\
    =& \mathbb{E}_{S_{1:T}}\mathbb{E}_{S_{1:T}'}\mathbb{E}_\mathcal{A}\Bigg[\frac{1}{n}\sum_{k=1}^n f(\mathcal{A}(S_{1:T}^{(i,k)});z_{i,k}')\Bigg] \\
    =& \mathbb{E}_{S_{1:T}}\mathbb{E}_{S_{1:T}'}\mathbb{E}_\mathcal{A}\Bigg[\frac{1}{n}\sum_{k=1}^n f(\mathcal{A}(S_{1:T});z_{i,k}')\Bigg] + \delta \\
    =& 
    \mathbb{E}_{S_{1:T}}\mathbb{E}_\mathcal{A}[R_{\mathcal{D}_i}(\mathcal{A}(S_{1:T}))] + \delta.
\end{aligned}
\end{equation*}
where
\[
\begin{aligned}
\delta=&\mathbb{E}_{S_{1:T}}\mathbb{E}_{S_{1:T}'}\mathbb{E}_\mathcal{A}\Bigg[
\frac{1}{n}\sum_{k=1}^n f(\mathcal{A}(S_{1:T}^{(i,k)});z_{i,k}')\\
&-\frac{1}{n}\sum_{k=1}^n f(\mathcal{A}(S_{1:T});z_{i,k}')\Bigg].
\end{aligned}
\]

Under the assumption that $\mathcal{A}$ is $\epsilon_{T,i}$-uniformly stable, we have
\[
|\delta|\leq \epsilon_{T,i}.
\]
Therefore,
\[
\left|\mathbb{E}_{S_{1:T},\mathcal{A}}\!\left[R_{\mathcal{D}_i}(\mathcal{A}(S_{1:T}))-R_{S_i}(\mathcal{A}(S_{1:T}))\right]\right|
= |\delta|
\leq \epsilon_{T,i}.
\]
\end{proof}

\subsection{Proof of Theorem \ref{Thm: CGG bound convex}}\label{Sec: proof of CGG bound convex}

\begin{theorem*}
Assume that the loss function $f(\cdot;z)$ satisfies Assumptions~\ref{assump: Lipschitz}, \ref{assump: smooth}, and \ref{assump: convex}. Suppose that SGD is run for $N$ steps on each task $\tau\in[T]$ with step sizes $0\leq\eta_{\tau,t}\leq \frac{2}{\beta}$ for $t=0,\ldots,N-1$. Then,
    \[\epsilon_{T,i}\leq\frac{2L^2}{n}\sum_{t=0}^{N-1} \eta_{i,t}.\]
\end{theorem*}

\begin{proof}    
    Let $S_{1:T}$ and $S_{1:T}'$ be two datasets containing $T$ tasks, differing in only one example in task $i$. Couple the two runs using the same initialization and sampled indices. Consider the update sequences $G_{1,0},\dots,G_{1,N-1},\dots,G_{T,0},\dots,G_{T,N-1}$ and $G'_{1,0},\dots,G'_{1,N-1},\dots,G'_{T,0},\dots,G'_{T,N-1}$ induced by running SGD on $S_{1:T}$ and $S_{1:T}'$, respectively.
    
    First, fix an arbitrary example $z \in \mathcal{Z}$. By the $L$-Lipschitz continuity of $f(\cdot; z)$, we have
    \begin{equation}\label{eq:lipschitz_bound}
        \mathbb{E}\bigl|f(w_{T,N}; z)-f(w'_{T,N}; z)\bigr|
        \leq L\,\mathbb{E}[\delta_{T,N}],
    \end{equation}
    where $\delta_{T,N}=\|w_{T,N}-w'_{T,N}\|$. Under the CL framework, the initialization for each task is exactly the final model from the previous task, that is, $w_{k,0}=w_{k-1,N}$ and $w'_{k,0}=w'_{k-1,N}$ for any $k\in\{2,\dots,T\}$. Hence, $\delta_{k,0}=\delta_{k-1,N}$. To bound $\mathbb{E}[\delta_{T,N}]$, we now consider two cases.

    (i) When the task index satisfies $j\in\{i+1,i+2,\dots,T\}$, the corresponding dataset for task $j$ is identical in $S_{1:T}$ and $S_{1:T}'$. Therefore, $G_{j,t}=G'_{j,t}$ holds for any $j\in\{i+1,i+2,\dots,T\}$ and $t\in\{0,\ldots,N-1\}$. This allows us to apply the $1$-expansivity of $G_{j,t}$. Under this case, we have
    \begin{equation}\label{Eqn: task_T_recursion OOD}
    \begin{aligned}
        \mathbb{E}[\delta_{T,N}]
        &\leq \mathbb{E}[\delta_{T,N-1}]\leq\cdots\leq \mathbb{E}[\delta_{T,0}]\\
        &=\mathbb{E}[\delta_{T-1,N}]\leq\cdots\\
        &\leq \mathbb{E}[\delta_{i+1,0}]=\mathbb{E}[\delta_{i,N}].
    \end{aligned}
    \end{equation}

     (ii) During the $t$-th step of training on task $i$, with probability $1 - \frac{1}{n}$, the example selected by SGD is identical in both $S_i$ and $S_i'$. In this case, the gradient update operators are identical (i.e., $G_{i,t} = G'_{i,t}$), and we can use the $1$-expansivity of $G_{i,t}$ due to the convexity and smoothness of $f$. With probability $\frac{1}{n}$, the selected examples differ. In this scenario, we rely on the property that both $G_{i,t}$ and $G'_{i,t}$ are $\eta_{i,t} L$-bounded. 
    
    Combining these two cases, we obtain the following recurrence relation for $t=0,\ldots,N-1$:
    \begin{equation*}
    \begin{aligned}
        \mathbb{E}[\delta_{i,t+1}]
        &\leq \left(1 - \frac{1}{n}\right)\mathbb{E}[\delta_{i,t}]
        + \frac{1}{n}\mathbb{E}[\delta_{i,t}] + \frac{2\eta_{i,t}L}{n}\\
        &= \mathbb{E}[\delta_{i,t}] + \frac{2\eta_{i,t}L}{n}.
    \end{aligned}
    \end{equation*}
    
    Unraveling this recursion from step $N$ down to step $0$ within task $i$, we get:
    \begin{equation} \label{Eqn: task_i_recursion OOD}
        \mathbb{E}[\delta_{i,N}] \leq \mathbb{E}[\delta_{i,0}] + \sum_{t=0}^{N-1} \frac{2\eta_{i,t}L}{n}.
    \end{equation}
    
    Since the data perturbation occurs \textit{only} in the dataset for task $i$, the datasets for all prior tasks remain identical. Consequently, for any previous task $m \in \{1, \dots, i-1\}$ and any update $t \in \{0, \dots, N-1\}$, the operators $G_{m,t} = G'_{m,t}$. Therefore, we can apply the $1$-expansivity property recursively:
    \begin{equation*}
        0 \leq \mathbb{E}[\delta_{i,0}] \leq \mathbb{E}[\delta_{i-1,N}] \leq \dots \leq \mathbb{E}[\delta_{1,0}] = 0.
    \end{equation*}
    
    Substituting $\mathbb{E}[\delta_{i,0}]=0$ into Equation~\eqref{Eqn: task_i_recursion OOD} and combining it with Equation~\eqref{Eqn: task_T_recursion OOD}, we obtain
    \begin{equation*}
        \mathbb{E}[\delta_{T,N}] \leq \frac{2L}{n}\sum_{t=0}^{N-1} \eta_{i,t}.
    \end{equation*}
    
    Finally, plugging this bound back into Equation \eqref{eq:lipschitz_bound}, we arrive at the stability bound:
    \begin{equation*}
        \mathbb{E}|f(w_{T,N}; z) - f(w'_{T,N}; z)| \leq \frac{2L^2}{n}\sum_{t=0}^{N-1} \eta_{i,t}.
    \end{equation*}
\end{proof}

\subsection{Proof of Theorem \ref{Thm: CGG bound non-convex}}\label{Sec: proof of CGG bound non-convex}

\begin{lemma}\label{Lem: non-convex-upper-bound OOD}
    Assume that the loss function $f(\cdot;z)$ is nonnegative and $L$-Lipschitz for every $z$, and that $n\geq2$. Let $S_{1:T}$ and $S_{1:T}'$ be two datasets consisting of samples from $T$ tasks, differing in only one example in task $i$. Then, for every $z\in\mathcal{Z}$ and every $t_0\in\{0,1,\ldots,N\}$ under the random update rule, we have
    \[
    \begin{aligned}
    &\mathbb{E}\abs{f(w_T;z)-f(w_T';z)}\\
    \leq&\frac{t_0}{n}\sup_{w,z}f(w;z)
    +L\mathbb{E}[\|w_T-w_T'\|\mid\delta_{i,t_0}=0].
    \end{aligned}
    \]
    The same inequality holds under the random permutation rule for $0\leq t_0\leq N$ and $t_0<n$.
\end{lemma}

\begin{proof}
    Let $\mathcal{E}$ denote the event that $\delta_{i,t_0}=0$, where  $t_0$ is the step index within task $i$. Then, we have
    \begin{equation*}
    \begin{aligned}
        &\mathbb{E}\abs{f(w_T;z)-f(w_T';z)}\\
        =&\mathbb{P}\{\mathcal{E}\}\mathbb{E}[\abs{f(w_T;z)-f(w_T';z)}\mid\mathcal{E}]\\
        &+\mathbb{P}\{\mathcal{E}^c\}\mathbb{E}[\abs{f(w_T;z)-f(w_T';z)}\mid\mathcal{E}^c]\\
        \leq&\mathbb{E}[\abs{f(w_T;z)-f(w_T';z)}\mid\mathcal{E}]+\mathbb{P}\{\mathcal{E}^c\}\sup_{w,z}f(w;z)\\
        \leq&L\mathbb{E}[\|w_T-w_T'\|\mid\mathcal{E}]
        +\mathbb{P}\{\mathcal{E}^c\}\sup_{w,z}f(w;z).
    \end{aligned}    
    \end{equation*}
    The bound on $\mathbb{P}\{\mathcal{E}^c\}$ is consistent with Lemma~\ref{Lem: non-convex-upper-bound}.
\end{proof}

\begin{theorem*}
Assume that the loss function $f(\cdot;z)\in[0,1]$ satisfies Assumptions~\ref{assump: Lipschitz} and \ref{assump: smooth}, with $n\geq2$ and $c,\beta>0$. Suppose that SGD is run for $N$ steps on each task $\tau\in[T]$ with learning rates $\eta_{\tau,t}$ for $t=0,\ldots,N-1$. Then,\\
(a) under intra-task decaying learning rates $\eta_{\tau,t}=\frac{c}{t+1}$:
    \[
     \epsilon_{T,i}\leq\begin{cases}
         \mathcal{B}_2^1, &\text{if }2cL^2(eN)^{c\beta(T-i)}<N,\\[2ex]
          \displaystyle\frac{N}{n}, &\text{if }2cL^2(eN)^{c\beta(T-i)}\geq N.
     \end{cases}
    \]

    where 
    \[
        \begin{aligned}
        \mathcal{B}_2^1=&\mathcal{K}_N(eN)^{\frac{c\beta(T-i)}{c\beta+1}}-\frac{2L^2(eN)^{c\beta(T-i)}}{\beta n}.
        \end{aligned}
    \]
(b) under continuously decaying learning rates $\eta_{\tau,t}=\frac{c}{(\tau-1)N+t+1}$ and $i>1$:
    {\small
    \[
    \epsilon_{T,i}
    \leq
    \begin{cases}
      \frac{2L^2}{\beta n}
      \left[\left(\frac{T}{i-1}\right)^{c\beta}-\left(\frac{T}{i}\right)^{c\beta}\right],&\text{if }2cL^2\leq\mathcal{C}_2^{\mathrm{CD}},\\[2ex]
      \mathcal{B}_2^2,
      &\text{if }\mathcal{C}_2^{\mathrm{CD}}<2cL^2<iN\left(\frac{i}{T}\right)^{c\beta},\\[2ex]
      \displaystyle\frac{N}{n}, &\text{if }2cL^2\geq iN\left(\frac{i}{T}\right)^{c\beta}.
    \end{cases}
    \]}
    where 
    {
    \[
    \begin{aligned}
        \mathcal{B}_2^2&=\mathcal{K}_{TN}-\frac{(i-1)N}{n}-\frac{2L^2}{\beta n}\left(\frac{T}{i}\right)^{c\beta},\\
        \mathcal{C}_2^{\mathrm{CD}}&=(i-1)N\left(\frac{i-1}{T}\right)^{c\beta}.
    \end{aligned}
    \]}
For $i=1$ under the continuously decaying schedule, 
\[
\epsilon_{T,1}\leq\begin{cases}
    \mathcal{K}_{TN}-\frac{2L^2T^{c\beta}}{\beta n},
    &\text{if }2cL^2T^{c\beta}<N,\\[2ex]
    \displaystyle\frac{N}{n},&\text{if }2cL^2T^{c\beta}\geq N.
\end{cases}
\]

\end{theorem*}
\begin{proof}
Let $S_{1:T}$ and $S_{1:T}'$ be two datasets containing $T$ tasks, differing in only one example in task $i$. Couple the two runs using the same initialization and sampled indices. Consider the update sequences $G_{1,0},\dots,G_{1,N-1},\dots,G_{T,0},\dots,G_{T,N-1}$ and $G'_{1,0},\dots,G'_{1,N-1},\dots,G'_{T,0},\dots,G'_{T,N-1}$ induced by running SGD on $S_{1:T}$ and $S_{1:T}'$, respectively.

By Lemma~\ref{Lem: non-convex-upper-bound OOD}, for any integer $t_0\in\{0,\ldots,N\}$, we have
\begin{equation}\label{Eqn: upper_bound_non_convex}
    \mathbb{E}\left|f(w_T;z)-f(w_T';z)\right|
    \leq \frac{t_0}{n} + L\,\mathbb{E}\!\left[\|w_T-w_T'\| \mid \delta_{i,t_0}=0\right].
\end{equation}

Define $\Delta_{\tau,t}:=\mathbb{E}[\delta_{\tau,t}\mid \delta_{i,t_0}=0]$. We first bound $\Delta_{T,N}$. By construction, for every task $m\in\{i+1,\dots,T\}$, the sampled examples at each step are identical in $S_{1:T}$ and $S_{1:T}'$. Hence, $G_{m,t}=G'_{m,t}$ for all $t\in\{0,\dots,N-1\}$, and we may use the $(1+\eta_{m,t}\beta)$-expansivity of $G_{m,t}$.

Therefore, for any $t\in\{0,\dots,N-1\}$,

(a) under intra-task decaying learning rates $\eta_{\tau,t}=\frac{c}{t+1}$,
\begin{equation*}
\begin{aligned}
    \Delta_{T,t+1}
    &\leq (1+\eta_{T,t}\beta)\Delta_{T,t} \\
    &= \left(1+\frac{c\beta}{t+1}\right)\Delta_{T,t} \\
    &\leq \exp\!\left(\frac{c\beta}{t+1}\right)\Delta_{T,t}.
\end{aligned}
\end{equation*}

Unrolling this recursion from $N$ down to $0$ yields
\begin{equation*}
\begin{aligned}
    \Delta_{T,N}
    &\leq \prod_{k=0}^{N-1} \exp\!\left(\frac{c\beta}{k+1}\right)\Delta_{T,0} \\
    &= \exp\!\left(c\beta\sum_{k=0}^{N-1} \frac{1}{k+1}\right)\Delta_{T,0} \\
    &= \exp\!\left(c\beta\sum_{k=1}^{N} \frac{1}{k}\right)\Delta_{T,0} \\
    &\leq \exp\!\left(c\beta(\ln N+1)\right)\Delta_{T,0} \\
    &= (eN)^{c\beta}\Delta_{T,0}.
\end{aligned}
\end{equation*}

Under the CL protocol, the initialization of each task is the final model from the previous task, namely $w_{k,0}=w_{k-1,N}$ and $w'_{k,0}=w'_{k-1,N}$ for any $k\in\{2,\ldots,T\}$. Hence, $\delta_{k,0}=\delta_{k-1,N}$. Therefore,
\[
\begin{aligned}
\Delta_{T,N}
&\leq (eN)^{c\beta}\Delta_{T,0}\\
&= (eN)^{c\beta}\Delta_{T-1,N}\\
&\leq \cdots\leq (eN)^{c\beta(T-i)}\Delta_{i,N}.
\end{aligned}
\]
For $1\leq t_0\leq N$, Equation~\eqref{Eqn: Delta_iN_ID_LR} further gives
\[
\Delta_{T,N}\leq\frac{2L(eN)^{c\beta(T-i)}}{\beta n}
\left[\left(\frac{N}{t_0}\right)^{c\beta}-1\right].
\]

Substituting this bound into Equation~\eqref{Eqn: upper_bound_non_convex}, we have
\begin{equation*}
\begin{aligned}
    &\mathbb{E}\left|f(w_T;z)-f(w_T';z)\right|\\
    \leq& \frac{t_0}{n}
    +\frac{2L^2(eN)^{c\beta(T-i)}}{\beta n}
    \left[\left(\frac{N}{t_0}\right)^{c\beta}-1\right].
\end{aligned}
\end{equation*}

Differentiating the right-hand side with respect to a positive real cutoff gives the stationary point
\[
\widehat t_0=
(2cL^2)^{\frac{1}{c\beta+1}} e^{\frac{c\beta(T-i)}{c\beta+1}} N^{\frac{c\beta(T-i+1)}{c\beta+1}}.
\]
For integer SGD steps, apply Equation~\eqref{Eqn: integer-cutoff comparison} to the unrolled sum in Equation~\eqref{Eqn: Delta_iN_ID_LR}, after using $(N/t)^{(1-1/n)c\beta}\leq(N/t)^{c\beta}$ and the future-task factor. The corresponding function is
\[
D(x)=2cL^2(eN)^{c\beta(T-i)}N^{c\beta}x^{-c\beta-1}.
\]
When $0<\widehat t_0<N$, evaluating the finite integral on both sides of $\widehat t_0$ yields
\[
\begin{aligned}
&\mathbb{E}\left|f(w_T;z)-f(w_T';z)\right|\\
\leq&\frac{\widehat t_0}{n}
+\frac{2L^2(eN)^{c\beta(T-i)}}{\beta n}
\left[\left(\frac{N}{\widehat t_0}\right)^{c\beta}-1\right]\\
=&\frac{c\beta+1}{nc\beta}\widehat t_0
-\frac{2L^2(eN)^{c\beta(T-i)}}{\beta n}.
\end{aligned}
\]
This evaluation also covers $\widehat t_0<1$, since the comparison in Equation~\eqref{Eqn: integer-cutoff comparison} permits a zero integer cutoff. When $\widehat t_0\geq N$, the integral equals $N$, giving the bound $N/n$. The condition $\widehat t_0<N$ is equivalent to $2cL^2(eN)^{c\beta(T-i)}<N$, which proves part (a). The two expressions coincide at equality. The case $L=0$ follows directly.

(b) under continuously decaying learning rates $\eta_{\tau,t}=\frac{c}{(\tau-1) N+t+1}$,
\begin{equation*}
\begin{aligned}
    \Delta_{T,t+1}
    &\leq (1+\eta_{T,t}\beta)\Delta_{T,t} \\
    &= \left(1+\frac{c\beta}{(T-1)N+t+1}\right)\Delta_{T,t} \\
    &\leq \exp\!\left(\frac{c\beta}{(T-1)N+t+1}\right)\Delta_{T,t}.
\end{aligned}
\end{equation*}

For any $j\in\{i+1,\dots,T\}$, unrolling this recursion from $N$ down to $0$ yields
\begin{equation*}
\begin{aligned}
    \Delta_{j,N}
    &\leq \prod_{k=0}^{N-1} \exp\!\left(\frac{c\beta}{(j-1)N+k+1}\right)\Delta_{j,0} \\
    &= \exp\!\left(c\beta\sum_{k=0}^{N-1} \frac{1}{(j-1)N+k+1}\right)\Delta_{j,0} \\
    &= \exp\!\left(c\beta\sum_{k=1}^{N} \frac{1}{(j-1)N+k}\right)\Delta_{j,0} \\
    &\leq \exp\!\left(c\beta\ln\!\left(\frac{j}{j-1}\right)\right)\Delta_{j,0} \\
    &= \left(\frac{j}{j-1}\right)^{c\beta}\Delta_{j,0}.
\end{aligned}
\end{equation*}
Under the CL protocol, the initialization of each task is the final model from the previous task, namely $w_{k,0}=w_{k-1,N}$ and $w'_{k,0}=w'_{k-1,N}$ for any $k\in\{2,\ldots,T\}$. Hence, $\delta_{k,0}=\delta_{k-1,N}$. Therefore,
\begin{equation*}
\begin{aligned}
\Delta_{T,N}
\leq& \left(\frac{T}{T-1}\right)^{c\beta}\Delta_{T,0}\\
=& \left(\frac{T}{T-1}\right)^{c\beta}\Delta_{T-1,N}\\
\leq& \cdots\\
\leq& \left(\frac{T}{T-1}\right)^{c\beta}\cdots\left(\frac{i+1}{i}\right)^{c\beta}\Delta_{i,N}\\
=&\left(\frac{T}{i}\right)^{c\beta}\Delta_{i,N}.
\end{aligned}
\end{equation*}

Based on Equation~\eqref{Eqn: Delta_iN_CD_LR}, we can further obtain
\[
\Delta_{T,N}\leq\frac{2L}{\beta n}
\left[\left(\frac{TN}{(i-1)N+t_0}\right)^{c\beta}
-\left(\frac{T}{i}\right)^{c\beta}\right].
\]

Substituting this bound into Equation~\eqref{Eqn: upper_bound_non_convex}, we have
\begin{equation}\label{Eqn:bound-CGG-CD}
\begin{aligned}
    &\mathbb{E}\left|f(w_T;z)-f(w_T';z)\right|\\
    \leq& \frac{t_0}{n}+\frac{2L^2}{\beta n}
    \left[\left(\frac{TN}{(i-1)N+t_0}\right)^{c\beta}
    -\left(\frac{T}{i}\right)^{c\beta}\right].
\end{aligned}
\end{equation}

Differentiating the right-hand side gives the stationary point of its continuous relaxation:
\[
\widehat t_0=(2cL^2)^{\frac{1}{c\beta+1}} (TN)^{\frac{c\beta}{c\beta+1}} - (i-1)N.
\]
The integer cutoff is again justified by Equation~\eqref{Eqn: integer-cutoff comparison}, applied to the finite-sum recursion with
\[
D(x)=2cL^2(TN)^{c\beta}\big((i-1)N+x\big)^{-c\beta-1}.
\]
If $\widehat t_0\leq0$, choose $t_0=0$ in Equation~\eqref{Eqn:bound-CGG-CD}. This gives
\[
\mathbb{E}\left|f(w_T;z)-f(w_T';z)\right|\leq\frac{2L^2}{\beta n}
\left[\left(\frac{T}{i-1}\right)^{c\beta}-\left(\frac{T}{i}\right)^{c\beta}\right].
\]
If $0<\widehat t_0<N$, evaluating the finite integral at this threshold gives
\[
\begin{aligned}
&\mathbb{E}\left|f(w_T;z)-f(w_T';z)\right|\\
\leq&\frac{c\beta+1}{nc\beta}\big((i-1)N+\widehat t_0\big)-\frac{(i-1)N}{n}
-\frac{2L^2}{\beta n}\left(\frac{T}{i}\right)^{c\beta}.
\end{aligned}
\]
If $\widehat t_0\geq N$, the integral equals $N$, yielding $N/n$. The conditions $\widehat t_0\leq0$ and $\widehat t_0\geq N$ are respectively equivalent to $2cL^2\leq(i-1)N((i-1)/T)^{c\beta}$ and $2cL^2\geq iN(i/T)^{c\beta}$. These cases establish part (b), and the adjacent formulas agree at both boundaries. For $i=1$, the local training is the same as in part (a), while the future-task factor is $T^{c\beta}$. Substituting this factor into the analysis of part (a) gives the stated two-branch bound for the continuously decaying schedule.

\end{proof}

\section{Empirical Forgetting} \label{Sec: proof of EF}

\begin{theorem*}
    Assume that the loss function $f(\cdot;z)$ satisfies Assumptions~\ref{assump: Lipschitz} and~\ref{assump: smooth}. Suppose that SGD is run for $N$ steps on each task $\tau\in[T]$ with learning rates $\eta_{\tau,t}$. Then,\\
    (a) under constant learning rates $\eta_{\tau,t}=c$:
    \begin{align*}
        \tilde{\epsilon}_{T,i} \leq -\mathcal{G}_{T,i}+\frac{\beta c^2L^2}{2}N(T-i);
    \end{align*}
    (b) under intra-task decaying learning rates $\eta_{\tau,t}=\frac{c}{t+1}$:
    \begin{align*}
        \tilde{\epsilon}_{T,i} \leq -\mathcal{G}_{T,i}+\frac{\beta c^2L^2}{2}\left(2-\frac1N\right)(T-i);
    \end{align*}
    (c) under continuously decaying learning rates $\eta_{\tau,t}=\frac{c}{(\tau-1)N+t+1}$:
    \begin{align*}
        \tilde{\epsilon}_{T,i} \leq -\mathcal{G}_{T,i}+\frac{\beta c^2L^2}{2}\left(\frac1{iN}-\frac1{TN}\right).
    \end{align*}
\end{theorem*}

\begin{proof}
We first consider the impact on task $i$ of performing one training step on task $\tau>i$, and then accumulate the empirical risk changes to obtain the final bound. Let $\mathcal{F}_{\tau,t}$ contain all training datasets, the initialization, and the sampling indices before the current update. Since the current index $i_t$ is sampled independently and uniformly from $[n]$, we have
\begin{align*}
&\mathbb{E}\!\left[\nabla f(w_{\tau,t};z_{\tau,i_t})\mid\mathcal{F}_{\tau,t}\right]
=\nabla R_{S_\tau}(w_{\tau,t}),\\
&\mathbb{E}\!\left[\|\nabla f(w_{\tau,t};z_{\tau,i_t})\|^2\mid\mathcal{F}_{\tau,t}\right]
=\frac1n\sum_{j=1}^n\|\nabla f(w_{\tau,t};z_{\tau,j})\|^2.
\end{align*}
The function $R_{S_i}$ is $\beta$-smooth. Applying the smoothness inequality to the SGD update and taking the conditional expectation gives
\begin{align*}
&\mathbb{E}\!\left[R_{S_i}(w_{\tau,t+1})-R_{S_i}(w_{\tau,t})\mid\mathcal{F}_{\tau,t}\right]\\
\leq&-\eta_{\tau,t}\left\langle\nabla R_{S_i}(w_{\tau,t}),\nabla R_{S_\tau}(w_{\tau,t})\right\rangle\\
&\quad+\frac{\beta\eta_{\tau,t}^2}{2n}\sum_{j=1}^n\|\nabla f(w_{\tau,t};z_{\tau,j})\|^2\\
\leq&-\eta_{\tau,t}\left\langle\nabla R_{S_i}(w_{\tau,t}),\nabla R_{S_\tau}(w_{\tau,t})\right\rangle+\frac{\beta L^2\eta_{\tau,t}^2}{2},
\end{align*}
where the last inequality uses the gradient bound implied by Assumption~\ref{assump: Lipschitz}. Taking the full expectation and accumulating from $t=0$ to $N-1$, we have
\begin{align*}
&\mathbb{E}\left[R_{S_i}(w_\tau)-R_{S_i}(w_{\tau-1})\right]\\
\leq&-\sum_{t=0}^{N-1}\eta_{\tau,t}\,\mathbb{E}\left\langle\nabla R_{S_i}(w_{\tau,t}),\nabla R_{S_\tau}(w_{\tau,t})\right\rangle+\frac{\beta L^2}{2}\sum_{t=0}^{N-1}\eta_{\tau,t}^2.
\end{align*}
Then, summing over the subsequent tasks gives
\begin{equation}\label{Eqn: EF general bound}
\tilde{\epsilon}_{T,i}
\leq-\mathcal{G}_{T,i}+\frac{\beta L^2}{2}\sum_{\tau=i+1}^T\sum_{t=0}^{N-1}\eta_{\tau,t}^2.
\end{equation}
We next discuss the different learning rates. Under constant learning rates $\eta_{\tau,t}=c$, we have
\begin{align*}
\tilde{\epsilon}_{T,i}\leq-\mathcal{G}_{T,i}+\frac{\beta c^2L^2}{2}N(T-i).
\end{align*}
Under intra-task decaying learning rates $\eta_{\tau,t}=c/(t+1)$, we have
\begin{align*}
\tilde{\epsilon}_{T,i}
&\leq-\mathcal{G}_{T,i}+\frac{\beta c^2L^2}{2}(T-i)\sum_{t=0}^{N-1}\frac1{(t+1)^2}\\
&\leq-\mathcal{G}_{T,i}+\frac{\beta c^2L^2}{2}(T-i)\left(1+\int_1^N\frac{dx}{x^2}\right)\\
&=-\mathcal{G}_{T,i}+\frac{\beta c^2L^2}{2}\left(2-\frac1N\right)(T-i).
\end{align*}
Under continuously decaying learning rates $\eta_{\tau,t}=c/((\tau-1)N+t+1)$, we have
\begin{align*}
\tilde{\epsilon}_{T,i}
&\leq-\mathcal{G}_{T,i}+\frac{\beta c^2L^2}{2}\sum_{k=iN+1}^{TN}\frac1{k^2}\\
&\leq-\mathcal{G}_{T,i}+\frac{\beta c^2L^2}{2}\int_{iN}^{TN}\frac{dx}{x^2}\\
&=-\mathcal{G}_{T,i}+\frac{\beta c^2L^2}{2}\left(\frac1{iN}-\frac1{TN}\right).
\end{align*}
This completes the proof.
\end{proof}

\section{Forgetting in Expectation} \label{Sec: Forgetting}
\subsection{Proof of Theorem~\ref{Thm: FT-convex}}
\begingroup
\allowdisplaybreaks[1]
\interlinepenalty=10000
\begin{proof}
By Equation~\eqref{Eqn: forgetting decomposed} and Theorems~\ref{Thm: IGG} and~\ref{Thm: CGG},
\[
F_T\leq\frac1{T-1}\sum_{i=1}^{T-1}\left(\epsilon_{i,i}+\epsilon_{T,i}+\tilde\epsilon_{T,i}\right).
\]
We apply the convex stability bounds and the EF bound under each learning-rate schedule.

(a) under constant learning rates $\eta_{\tau,t}=c$:
\begin{align*}
F_T\leq&\frac1{T-1}\sum_{i=1}^{T-1}(\epsilon_{i,i}+\epsilon_{T,i}+\tilde\epsilon_{T,i})\\
\leq&-\frac1{T-1}\sum_{i=1}^{T-1}\mathcal{G}_{T,i}\\
&+\frac1{T-1}\sum_{i=1}^{T-1}\left[\frac{\beta c^2L^2}{2}N(T-i)+\frac{4cL^2N}{n}\right]\\
=&-\frac1{T-1}\sum_{i=1}^{T-1}\mathcal{G}_{T,i}\\
&+\frac1{T-1}\left[\frac{\beta c^2L^2N}{2}\frac{T(T-1)}2+\frac{4cL^2N(T-1)}n\right]\\
=&-\frac1{T-1}\sum_{i=1}^{T-1}\mathcal{G}_{T,i}+\frac{\beta c^2L^2NT}{4}+\frac{4cL^2N}{n}.
\end{align*}

(b) under intra-task decaying learning rates $\eta_{\tau,t}=\frac{c}{t+1}$:
\begin{align*}
F_T\leq
&\frac1{T-1}\sum_{i=1}^{T-1}(\epsilon_{i,i}+\epsilon_{T,i}+\tilde\epsilon_{T,i})\\
\leq&-\frac1{T-1}\sum_{i=1}^{T-1}\mathcal{G}_{T,i}+\frac1{T-1}\sum_{i=1}^{T-1}\Biggl[\frac{\beta c^2L^2}{2}\left(2-\frac1N\right)\\
&\times(T-i)+\frac{4cL^2\ln(eN)}n\Biggr]\\
=&-\frac1{T-1}\sum_{i=1}^{T-1}\mathcal{G}_{T,i}+\frac1{T-1}\Biggl[\frac{\beta c^2L^2}{2}\left(2-\frac1N\right)\\
&\times\frac{T(T-1)}2+\frac{4cL^2\ln(eN)(T-1)}n\Biggr]\\
=&-\frac1{T-1}\sum_{i=1}^{T-1}\mathcal{G}_{T,i}+\frac{\beta c^2L^2T}{4}\left(2-\frac1N\right)\\
&+\frac{4cL^2\ln(eN)}{n}.
\end{align*}

For the continuous schedule, we use the finite harmonic-sum estimate
\begin{equation}\label{Eqn: FT harmonic estimate}
\begin{aligned}
\sum_{i=1}^{T-1}\frac1i-\frac{T-1}{T}
&\leq1+\int_1^{T-1}\frac{du}{u}-\frac{T-1}{T}\\
&=\ln(T-1)+\frac1T.
\end{aligned}
\end{equation}
This also holds for $T=2$, when the integral is zero.

(c) under continuously decaying learning rates $\eta_{\tau,t}=\frac{c}{(\tau-1)N+t+1}$:
\begin{align*}
&F_T\\
\leq&\frac1{T-1}\sum_{i=1}^{T-1}(\epsilon_{i,i}+\epsilon_{T,i}+\tilde\epsilon_{T,i})\\
\leq&-\frac1{T-1}\sum_{i=1}^{T-1}\mathcal{G}_{T,i}+\frac1{T-1}\Biggl[\sum_{i=1}^{T-1}\frac{\beta c^2L^2}{2N}\left(\frac1i-\frac1T\right)\\
&+\frac{4cL^2\ln(eN)}n+\sum_{i=2}^{T-1}\frac{4cL^2}n\ln\!\left(1+\frac1{i-1}\right)\Biggr]\\
=&-\frac1{T-1}\sum_{i=1}^{T-1}\mathcal{G}_{T,i}+\frac1{T-1}\Biggl[\frac{\beta c^2L^2}{2N}\left(\sum_{i=1}^{T-1}\frac1i-\frac{T-1}{T}\right)\\
&+\frac{4cL^2\ln(eN)}n+\frac{4cL^2}n\ln(T-1)\Biggr]\\
\leq&-\frac1{T-1}\sum_{i=1}^{T-1}\mathcal{G}_{T,i}+\frac{\beta c^2L^2}{2N(T-1)}\left(\ln(T-1)+\frac1T\right)\\
&+\frac{4cL^2}{n(T-1)}\ln\big(eN(T-1)\big).
\end{align*}
Here, the generalization terms use $\sum_{t=1}^N1/t\leq\ln(eN)$ for task $1$ and the telescoping identity
\[
\sum_{i=2}^{T-1}\ln\!\left(\frac{i}{i-1}\right)=\ln(T-1).
\]
The last inequality follows from Equation~\eqref{Eqn: FT harmonic estimate}. When $T=2$, the logarithmic sum is empty, so the same calculation applies. This completes the proof.
\end{proof}
\endgroup

\subsection{Proof of Theorem~\ref{Thm: FT-non-convex}}
\begin{proof}
By Equation~\eqref{Eqn: forgetting decomposed} and the generalization bounds,
\[
F_T\leq\frac1{T-1}\sum_{i=1}^{T-1}\left(\epsilon_{i,i}+\epsilon_{T,i}+\tilde\epsilon_{T,i}\right).
\]
We evaluate these three contributions under each learning-rate schedule.

(a) Under intra-task decaying learning rates $\eta_{\tau,t}=\frac{c}{t+1}$, Theorem~\ref{Thm: EF} gives
\[
\begin{aligned}
&\frac1{T-1}\sum_{i=1}^{T-1}\tilde\epsilon_{T,i}\\
\leq&-\frac1{T-1}\sum_{i=1}^{T-1}\mathcal{G}_{T,i}+\frac{\beta c^2L^2}{2(T-1)}\left(2-\frac1N\right)
\sum_{i=1}^{T-1}(T-i)\\
=&-\frac1{T-1}\sum_{i=1}^{T-1}\mathcal{G}_{T,i}
+\frac{\beta c^2L^2T}{4}\left(2-\frac1N\right).
\end{aligned}
\]
Under the condition $2cL^2(eN)^{c\beta(T-1)}<N$, for every $i\in[T-1]$ we have
\[
\begin{aligned}
2cL^2&\leq2cL^2(eN)^{c\beta(T-i)}\leq2cL^2(eN)^{c\beta(T-1)}<N.
\end{aligned}
\]
Thus, the interior expressions in Theorems~\ref{Thm: IGG bound non-convex} and~\ref{Thm: CGG bound non-convex} apply to all historical tasks. The IGG bound is the same for every task, so
\[
\begin{aligned}
\frac1{T-1}\sum_{i=1}^{T-1}\epsilon_{i,i}
\leq\frac{c\beta+1}{nc\beta}(2cL^2)^{\frac1{c\beta+1}}N^{\frac{c\beta}{c\beta+1}}-\frac{2L^2}{\beta n}.
\end{aligned}
\]
For CGG, setting $j=T-i$ and summing the interior expressions gives
\[
\begin{aligned}
\sum_{i=1}^{T-1}\epsilon_{T,i}
\leq{}&\frac{c\beta+1}{nc\beta}(2cL^2)^{\frac1{c\beta+1}}N^{\frac{c\beta}{c\beta+1}}\sum_{j=1}^{T-1}(eN)^{\frac{c\beta j}{c\beta+1}}\\
&-\frac{2L^2}{\beta n}\sum_{j=1}^{T-1}(eN)^{c\beta j}.
\end{aligned}
\]
The two geometric sums are
\[
\begin{aligned}
\sum_{j=1}^{T-1}(eN)^{\frac{c\beta j}{c\beta+1}}
&=\frac{(eN)^{\frac{c\beta(T-1)}{c\beta+1}}-1}{1-(eN)^{-\frac{c\beta}{c\beta+1}}},\\
\sum_{j=1}^{T-1}(eN)^{c\beta j}
&=\frac{(eN)^{c\beta(T-1)}-1}{1-(eN)^{-c\beta}}.
\end{aligned}
\]
Both denominators are positive since $eN>1$ and $c\beta>0$. Dividing the CGG sum by $T-1$, adding the IGG and EF contributions, and collecting the common coefficients gives Equation~\eqref{Eqn: FT ID explicit}. When $T=2$, each geometric sum contains one term, and the same identities apply.

(b) Under continuously decaying learning rates $\eta_{\tau,t}=\frac{c}{(\tau-1)N+t+1}$:

The unrolled CGG recursion in Appendix~\ref{Sec: proof of CGG bound non-convex} gives, for every integer $t_0\in\{0,\ldots,N\}$,
\[
\begin{aligned}
\epsilon_{T,i}\leq\frac1n\Biggl[t_0+2cL^2(TN)^{c\beta}
\sum_{k=t_0+1}^{N}\frac1{((i-1)N+k)^{c\beta+1}}\Biggr].
\end{aligned}
\]
The summand is nonnegative and nonincreasing in $k$. Applying Equation~\eqref{Eqn: integer-cutoff comparison} and substituting $v=(i-1)N+u$ therefore gives, for each $i\in[T-1]$,
\[
\epsilon_{T,i}\leq\frac1n\int_{(i-1)N}^{iN}
\min\left\{1,\frac{2cL^2(TN)^{c\beta}}{v^{c\beta+1}}\right\}\,dv.
\]
The integrand is the same for all tasks. Summing the integrals over their adjacent intervals yields
\[
\begin{aligned}
&\sum_{i=1}^{T-1}\epsilon_{T,i}\leq\frac1n\int_0^{(T-1)N}
\min\left\{1,\frac{2cL^2(TN)^{c\beta}}{v^{c\beta+1}}\right\}\,dv.
\end{aligned}
\]
The two arguments of the minimum are equal at
\[
\widehat v_0=(2cL^2)^{\frac1{c\beta+1}}(TN)^{\frac{c\beta}{c\beta+1}}.
\]
If $2cL^2T^{c\beta}<N(T-1)^{c\beta+1}$, then $\widehat v_0<(T-1)N$. Splitting the integral at $\widehat v_0$ gives
\[
\begin{aligned}
\sum_{i=1}^{T-1}\epsilon_{T,i}
\leq&\frac{\widehat v_0}n+\frac{2cL^2(TN)^{c\beta}}n
\int_{\widehat v_0}^{(T-1)N}v^{-c\beta-1}\,dv\\
=&\frac{\widehat v_0}n+\frac{2L^2(TN)^{c\beta}}{\beta n}
\left[\widehat v_0^{-c\beta}-((T-1)N)^{-c\beta}\right]\\
=&\frac{c\beta+1}{nc\beta}(2cL^2)^{\frac1{c\beta+1}}(TN)^{\frac{c\beta}{c\beta+1}}-\frac{2L^2}{\beta n}\left(\frac{T}{T-1}\right)^{c\beta}.
\end{aligned}
\]
If $2cL^2T^{c\beta}\geq N(T-1)^{c\beta+1}$, then $\widehat v_0\geq(T-1)N$, so the integrand equals $1$ throughout the interval and
\[
\sum_{i=1}^{T-1}\epsilon_{T,i}\leq\frac{(T-1)N}n.
\]

For IGG, the same integer-cutoff comparison gives
\[
\epsilon_{1,1}\leq\frac1n\int_0^N
\min\left\{1,\frac{2cL^2N^{c\beta}}{u^{c\beta+1}}\right\}\,du.
\]
If $2cL^2<N$, the crossing point is $\widehat u_0=(2cL^2)^{1/(c\beta+1)}N^{c\beta/(c\beta+1)}<N$. Splitting the integral at this point gives
\[
\begin{aligned}
\epsilon_{1,1}\leq{}&\frac{\widehat u_0}n+\frac{2cL^2N^{c\beta}}n
\int_{\widehat u_0}^{N}u^{-c\beta-1}\,du\\
={}&\frac{c\beta+1}{nc\beta}(2cL^2)^{\frac1{c\beta+1}}N^{\frac{c\beta}{c\beta+1}}-\frac{2L^2}{\beta n}.
\end{aligned}
\]
If $2cL^2\geq N$, the integrand equals $1$ on $[0,N]$, giving $\epsilon_{1,1}\leq N/n$. For every $i\geq2$, choosing the integer cutoff $t_0=0$ in Equation~\eqref{Eqn:bound-IGG-CD} gives
\[
\epsilon_{i,i}\leq\frac{2L^2}{\beta n}
\left[\left(1+\frac1{i-1}\right)^{c\beta}-1\right].
\]
For $0\leq x\leq\ln2$, convexity of $e^{c\beta x}$ implies
\[
e^{c\beta x}-1\leq\frac{2^{c\beta}-1}{\ln2}\,x.
\]
Taking $x=\ln(1+1/k)$ and summing over $k=1,\ldots,T-2$, we obtain
\[
\begin{aligned}
\sum_{i=2}^{T-1}\epsilon_{i,i}
&\leq\frac{2L^2}{\beta n}
\sum_{k=1}^{T-2}\left[\left(1+\frac1k\right)^{c\beta}-1\right]\\
&\leq\frac{2L^2}{\beta n}\frac{2^{c\beta}-1}{\ln2}
\sum_{k=1}^{T-2}\ln\!\left(\frac{k+1}{k}\right)\\
&=\frac{2L^2}{\beta n}\frac{2^{c\beta}-1}{\ln2}\ln(T-1).
\end{aligned}
\]
For $T=2$, these sums are empty and both sides are zero. Thus, when $2cL^2<N$,
\[
\begin{aligned}
\sum_{i=1}^{T-1}\epsilon_{i,i}
\leq{}&\frac{c\beta+1}{nc\beta}(2cL^2)^{\frac1{c\beta+1}}N^{\frac{c\beta}{c\beta+1}}\\
&+\frac{2L^2}{\beta n}
\left[\frac{2^{c\beta}-1}{\ln2}\ln(T-1)-1\right].
\end{aligned}
\]
When $2cL^2\geq N$, adding the first-task endpoint bound instead gives
\[
\sum_{i=1}^{T-1}\epsilon_{i,i}
\leq\frac{N}{n}+\frac{2L^2}{\beta n}\frac{2^{c\beta}-1}{\ln2}\ln(T-1).
\]
Finally, the finite harmonic sum satisfies
\[
\begin{aligned}
\sum_{i=1}^{T-1}\frac1i-\frac{T-1}{T}
&\leq1+\int_1^{T-1}\frac{du}{u}-\frac{T-1}{T}\\
&=\ln(T-1)+\frac1T.
\end{aligned}
\]
Combining this inequality with Theorem~\ref{Thm: EF} gives
\[
\begin{aligned}
&\frac1{T-1}\sum_{i=1}^{T-1}\tilde\epsilon_{T,i}\\
\leq &-\frac1{T-1}\sum_{i=1}^{T-1}\mathcal{G}_{T,i}+\frac{\beta c^2L^2}{2N(T-1)}
\left(\ln(T-1)+\frac1T\right).
\end{aligned}
\]
The two IGG cases and the two cumulative CGG cases give the four combinations in Equation~\eqref{Eqn: FT CD explicit} after division by $T-1$ and addition of the EF bound. In particular, when both cutoff points reach their endpoints, their combined contribution is
\[
\frac1{T-1}\left(\frac{N}{n}+\frac{(T-1)N}{n}\right)=\frac{NT}{n(T-1)}.
\]
At $2cL^2=N$, the first-task interior expression equals $N/n$. At $2cL^2T^{c\beta}=N(T-1)^{c\beta+1}$, the cumulative CGG interior expression equals $(T-1)N/n$. Thus the expressions agree on each boundary, including their intersection. The empty IGG tail at $T=2$ was handled above, so all stated cases are covered. This completes the proof.
\end{proof}

\section{Experiments on Split MiniImageNet with ResNet-18}\label{Appendix:experiments-on-miniimagenet}

\begin{figure*}[t]
    \centering
    \begin{subfigure}{1\textwidth}
        \centering
        \includegraphics[width=\linewidth]{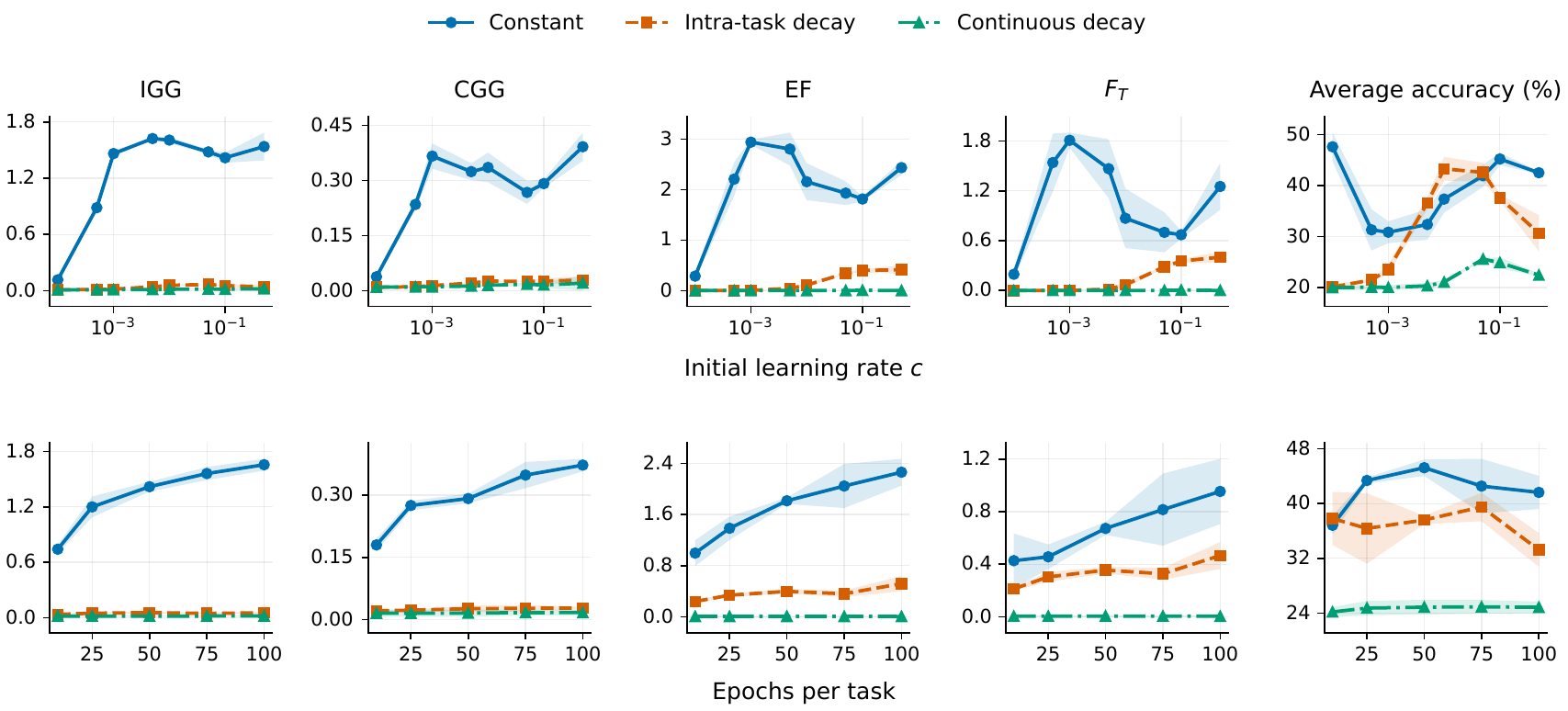}
    \end{subfigure}
    \caption{\textbf{IGG, CGG, EF, $F_T$, and ACC under different learning-rate schedules on the Split MiniImageNet benchmark.}
The top row shows results for different initial learning rates $c$, while the bottom row shows results for different numbers of training epochs per task. }
\label{fig:lr_epochs-mini}
\end{figure*}

\textbf{Experimental setting on Split MiniImageNet.  } The Split MiniImageNet dataset is constructed by dividing the 100 classes of MiniImageNet into 20 sequential tasks, each containing 5 classes. We use a ResNet-18 backbone, following a similar architecture to that used by \citet{lopez2017gradient}. We vary the number of epochs per task in $\{10,25,50,75,100\}$, the initial learning rate in $\{10^{-4},5\times10^{-4},10^{-3},5\times10^{-3},10^{-2},5\times10^{-2},10^{-1},5\times10^{-1}\}$, and the number of samples per task in $\{500,1000,1500,2000\}$. To examine the effect of the number of epochs per task, we fix $c=0.1$ and use the original training set setting. To examine the effect of the initial learning rate, we fix the number of epochs per task to 50 with the same training set setting. To examine the effect of the number of training samples per task, we fix $c=0.1$ and use a common budget of $N=1600$ updates per task across all sample sizes. Each experiment is repeated with random seeds 1, 2, and 3, and the figures report the mean and standard deviation across the three runs.  

\textbf{Measurement of cumulative gradient alignment.}  For each task, we construct a measurement buffer by uniformly sampling batches without replacement from its fixed training-set evaluation partition. We retain four batches on Split CIFAR-100 and two batches on Split MiniImageNet, with up to 64 samples per batch, giving maximum buffer sizes of 256 and 128 samples per task, respectively. At step $t$ of task $\tau$, let $g_{\tau,t}$ denote the gradient of the mean training loss on the mini-batch actually used for the current SGD update. For each previous task $i<\tau$, we uniformly sample one batch from its buffer to obtain a gradient estimate $\widehat{g}_{i,\tau,t}$ in evaluation mode using the corresponding task head. Both $\widehat{g}_{i,\tau,t}$ and the current training gradient $g_{\tau,t}$ are evaluated at the same parameters $w_{\tau,t}$ before the update. We then compute
\[
\widehat{\mathcal{G}}_{T,i}
=
\sum_{\tau=i+1}^{T}\sum_{t=0}^{N-1}
\eta_{\tau,t}
\left\langle \widehat{g}_{i,\tau,t},g_{\tau,t}\right\rangle.
\]
Each inner product is multiplied by the learning rate used for that update and accumulated over all subsequent updates, without normalization by gradient norms, the number of updates, or the number of tasks. The buffers are used exclusively for measurement. We plot $-\widehat{\mathcal{G}}_{T,i}$ alongside EF for each previous task $i$, reporting the mean and standard deviation across the three training seeds.

\begin{figure}[htbp]
    \centering
    \begin{subfigure}{0.485\textwidth}
        \centering
        \includegraphics[width=\linewidth]{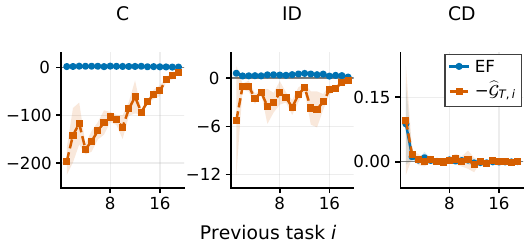}
    \end{subfigure}
    
    \caption{\textbf{EF and CGA under different learning-rate schedules on Split MiniImageNet.} C, ID, and CD denote constant, intra-task decaying, and continuously decaying learning rates, respectively. Curves show the means over three seeds, and shaded regions indicate one standard deviation.}
\label{fig:alignment-mini}
\end{figure}

\begin{figure}[htbp]
    \centering
    \begin{subfigure}[b]{0.485\textwidth}
        \centering
        \includegraphics[width=\linewidth]{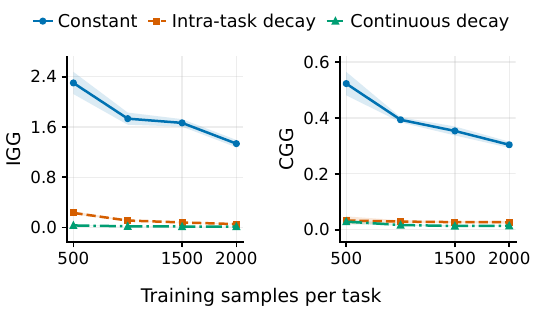}
    \end{subfigure}
   \caption{\textbf{IGG and CGG with different numbers of training samples per task on Split MiniImageNet.}}
\label{fig:samples-mini}
\end{figure}

\begin{figure}[htbp]
    \centering
    \begin{subfigure}[b]{0.485\textwidth}
        \centering
        \includegraphics[width=\linewidth]{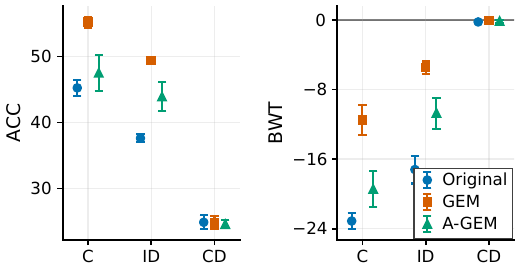}
    \end{subfigure}
    \caption{\textbf{ACC and BWT for original CL, GEM, and A-GEM under different learning-rate schedules on Split MiniImageNet.} Higher ACC indicates better predictive performance, while higher BWT indicates less forgetting.}
\label{fig:methods-mini}
\end{figure}

\end{document}